\documentclass[letterpaper]{article} 
\usepackage{aaai2027}  
\nocopyright
\usepackage[hyphens]{url}  
\usepackage{graphicx} 
\usepackage{natbib}  
\usepackage{caption} 
\usepackage{booktabs}
\usepackage{flafter}
\usepackage{listings}
\usepackage{xcolor}
\usepackage[percent]{overpic}

\usepackage{amsmath,amssymb,amsfonts}
\usepackage{amsthm}

\newcommand{\SCSE}{\textsc{SCSE}}
\newcommand{\Block}{\mathcal{B}}
\newcommand{\R}{\mathbb{R}}
\newcommand{\norm}[1]{\lVert#1\rVert}
\newcommand{\Tmap}{\mathcal{T}}

\newtheorem{theorem}{Theorem}
\newtheorem{lemma}[theorem]{Lemma}
\newtheorem{corollary}[theorem]{Corollary}
\newtheorem{proposition}[theorem]{Proposition}
\theoremstyle{remark}
\newtheorem*{remark}{Remark}

\definecolor{lstbg}{HTML}{F8F9FB}
\definecolor{lstframe}{HTML}{C7CDD4}
\definecolor{lstnumber}{HTML}{6B7280}
\definecolor{lstkeyword}{HTML}{1F4E79}
\definecolor{lstcomment}{HTML}{5F6F52}
\definecolor{lststring}{HTML}{8A4B1F}
\definecolor{lstfunction}{HTML}{7A3E9D}
\definecolor{lstconstant}{HTML}{0F766E}
\definecolor{lstidentifier}{HTML}{111827}
\definecolor{consistencybg}{HTML}{F3F8FF}
\definecolor{consistencyframe}{HTML}{2F5B89}
\definecolor{baselinefigure}{HTML}{8E2A20}

\newcommand{\consistencybox}[1]{%
  \par\smallskip
  \noindent\fcolorbox{consistencyframe}{consistencybg}{%
    \begin{minipage}{\dimexpr\linewidth-2\fboxsep-2\fboxrule\relax}
      \small #1
    \end{minipage}}%
  \par\smallskip
}

\lstdefinestyle{pythoncode}{
  language=Python,
  basicstyle=\ttfamily\scriptsize,
  keywordstyle=\color{lstkeyword}\bfseries,
  commentstyle=\color{lstcomment}\itshape,
  stringstyle=\color{lststring},
  identifierstyle=\color{lstidentifier},
  emph={scse_step,scse_unroll,pow,sum,clamp_min,to},
  emphstyle=\color{lstfunction},
  emph={[2]None,True,False},
  emphstyle={[2]\color{lstconstant}\bfseries},
  backgroundcolor=\color{lstbg},
  columns=fullflexible,
  keepspaces=true,
  showstringspaces=false,
  breaklines=true,
  breakatwhitespace=true,
  frame=single,
  framerule=0.4pt,
  framesep=3pt,
  rulecolor=\color{lstframe},
  numbers=left,
  numberstyle=\tiny\ttfamily\color{lstnumber},
  numbersep=4.25pt,
  xleftmargin=1.0em,
  framexleftmargin=1.55em,
  aboveskip=0.6\baselineskip,
  belowskip=0.4\baselineskip,
  captionpos=b,
  tabsize=4,
  upquote=true
}

\title{Looped Transformers with Source-Centered State Evolution}
\author{
	Bum Jun Kim\corresponding,
	Kohei Hayashi,
	Shunsuke Kamiya,\\
	Masanori Koyama,
	Yusuke Iwasawa,
	Yutaka Matsuo
}
\affiliations{
	The University of Tokyo\\
	\{bumjun.kim, kohei.hayashi, shunsuke.kamiya, masanori.koyama, iwasawa, matsuo\}@weblab.t.u-tokyo.ac.jp
}

\begin{document}

\maketitle

\begin{abstract}
	Looped Transformers create a useful train- and test-time compute axis by reusing the same Transformer block over recurrent depth, increasing effective depth at a fixed parameter count.  However, that shared block must then govern an entire trajectory of varying hidden states over trained and extrapolated depths.  Furthermore, in additive-injection looped Transformers, an input-conditioned signal is reintroduced at every recurrent step, so applying the shared transition at an input-conditioned reference can still move the hidden state.  In this paper, we propose Source-Centered State Evolution (\SCSE{}), which is designed to reconcile input conditioning with reference-preserving shared recurrence.  Specifically, \SCSE{} retains input dependence through its learned anchor and initial deviation, allows nonzero deviations to drive recurrent computation while mapping zero deviation to zero, and guarantees exact anchor invariance through its zero-deviation mask.  The designated anchor is thereby a one-step fixed point by construction.  The zero-deviation forcing bias is the next deviation produced from the anchor itself and vanishes in \SCSE{}, while nonzero deviations remain active and support state-dependent recurrent computation.  Our theory shows that the zero-deviation forcing bias is a design degree of freedom whose task effect can be harmful, neutral, or beneficial; \SCSE{} resolves this choice in favor of exact anchor invariance by setting the bias to zero.  Across WikiText-2, WikiText-103, direct web-corpus pretraining, held-out web-text transfer, LAMBADA completion, adaptive-depth evaluation, context lengths from 128 to 1024, and shared-block comparisons up to 139.2M parameters, \SCSE{} improves the controlled recurrent quality frontier.  Ablation studies identify the learned anchor and the anchor-coordinate deviation recurrence as the primary contributors to the gain, and a trained-model case study grounds the anchor-response diagnostic in observed recurrent motion.
\end{abstract}

\section{Introduction}
Modern Transformer language models (LMs) build on residual depth and self-attention \citep{he2016deep,vaswani2017attention} and are usually scaled by increasing parameters, data, or training compute \citep{kaplan2020scaling,hoffmann2022training}.  Recent test-time-scaling work also studies inference compute as a separate scaling axis \citep{snell2024scaling}.  A complementary architectural line is looped Transformers, which reuse parameters recurrently in latent space, increasing effective depth without increasing recurrent-block parameters.  This direction is promising because looped Transformers turn depth into a reusable inference-time compute knob.  At inference time, the model can apply the same trained weights for additional recurrent passes through the same shared block, potentially with adaptive halting or budget selection.  Additional recurrent passes can improve algorithmic generalization, latent reasoning, or length extrapolation without storing a separate parameter set for every layer.

The reusable depth axis, however, requires a single shared transition to govern an entire trajectory of varying hidden states over trained and extrapolated depths.  Furthermore, additive-injection looped Transformers reintroduce an input-conditioned signal whenever the shared block is applied.  We call this recurrently reused input-conditioned signal the source; repeatedly injecting the source leaves the shared transition's response at a chosen input-conditioned reference unconstrained.  Even at that reference, another application of the shared transition can produce a source-driven update.  Recurrent propagation can contract, cancel, exploit, or coherently accumulate this update over loop depth, thereby yielding a source-driven degree of freedom.  The possibility of such depth-varying propagation places an additional consistency burden on shared recurrence: the same transition must remain useful across the states generated along its own trajectory.

We make this source-driven degree of freedom precise.  Let $e$ denote the input representation and let $h^\star(e)$ be a reference state computed once from $e$.  We call $h^\star(e)$ an input-conditioned anchor.  Its value depends on the input but remains fixed throughout the recurrent unroll, serving as the origin for measuring recurrent motion.  Define the anchor-relative deviation $\Delta_t=h_t-h^\star(e)$, where $h_t$ denotes the hidden state at recurrent step $t$.  Let $\Tmap_t(\Delta;e)$ denote the model's actual one-step map in these coordinates.  Evaluating $\Tmap_t$ at zero deviation, corresponding to the input-conditioned anchor, gives the next deviation
\begin{equation}
	b_t(e):=\Tmap_t(0;e),
\end{equation}
which we call the zero-deviation forcing bias.  Thus, $b_t(e)=0$ exactly when zero deviation is a fixed point of $\Tmap_t(\cdot;e)$, equivalently when the chosen anchor is a one-step fixed point at recurrent step $t$.  The forcing bias $b_t(e)$ exposes a source-driven degree of freedom in additive recurrence.  The response can be attenuated, canceled, used for computation, or preserved across recurrent depth, while its effect on the task can be harmful, neutral, or beneficial depending on the readout and loss.  The appendix formalizes these regimes through an exact bias-subtraction counterfactual.

To resolve this design choice in favor of anchor invariance while retaining useful off-anchor motion, we propose Source-Centered State Evolution (\SCSE{}).  \SCSE{} builds a learned anchor once, evolves the deviation through a zero-preserving recurrent core, and uses a mask to make the zero-deviation threshold boundary exact.  The resulting map satisfies $\Tmap_t(0;e)=0$, making the designated anchor a one-step fixed point by construction, while nonzero deviations drive state-dependent recurrent computation.

More generally, this anchor-response diagnostic is not specific to looped Transformers.  Any repeatedly applied transition that reintroduces an input-conditioned source can have a nonzero response at a chosen reference.  Yet this design choice matters especially in looped Transformers because the same transition must handle both the anchor state and ordinary off-anchor states while repeatedly operating on its own outputs across trained and extrapolated depths.  Accordingly, we target the following design objective:

\consistencybox{\emph{A single shared transition should leave the chosen anchor unchanged, implement useful state-dependent motion away from that anchor, and maintain both properties along its self-generated trajectories within and beyond the training loop-depth range.}}

\noindent By contrast, conventional Transformers with unshared layers do not impose the same consistency condition on any one transition, because different layers can specialize to different depth indices and absorb additive source injection as part of a fixed-depth computation.  Under shared recurrence, however, the same source-conditioned transition repeatedly propagates the forcing response, so additive source injection becomes a recurrent forcing term rather than just a layer-local design choice.  This distinction motivates an explicit source-centered reparameterization.

Indeed, recent studies raise related concerns about unreliable dynamics under depth extrapolation \citep{prairie2026parcae,yang2026stars,park2026loopus,sharma2026readout}.  To our knowledge, however, they do not define, measure, or architecturally control the shared transition's response at an input-conditioned reference.  This paper isolates that quantity as the zero-deviation forcing bias $b_t(e)$.

Empirically, in the main WikiText-103 comparisons, \SCSE{} attains the lowest shared-block perplexity (PPL) at every evaluated loop depth, inside and beyond the training loop-depth range, among the looped Transformer baseline and the tuned, capacity-matched, and recurrent-step-conditioned controls.  Figure~\ref{fig:wt103-depth-response} and Table~\ref{tab:stepcond_wt103} report this comparison.  Matched controls and controlled ablation studies identify the learned anchor and the anchor-coordinate deviation recurrence as the main contributors to this advantage, as shown in Table~\ref{tab:add-anchor-design} and Appendix Table~\ref{tab:add-causal-forcing}.  To examine whether the anchor response targeted by the design is observed in trained models, Table~\ref{tab:bias_decomp} verifies the expected pointwise architectural contrast across trained WikiText-103 model families.  Additive-injection models retain nonzero anchor responses, whereas \SCSE{} realizes exact anchor invariance while retaining nonzero anchor-relative motion.  Complementing this replicated measurement, Figure~\ref{fig:zdfb-empirical} provides an anchor-started WikiText-2 case study in which forcing-aligned trajectories and dose-aligned displacements show that the source channel can steer recurrent motion.  Together, results from the performance comparisons, ablations, and diagnostics support source-centered recurrence as a useful design principle.

We make three contributions.  First, we define and measure the anchor-dependent zero-deviation forcing bias and theoretically show that the finite-horizon effect of this design degree of freedom on task loss can be harmful, neutral, or beneficial, depending on recurrent propagation and readout--loss alignment.  Second, we propose \SCSE{}, which evolves deviations around an input-conditioned anchor through a zero-preserving recurrent core, making the anchor a one-step fixed point while retaining useful off-anchor recurrent computation.  Third, across diverse evaluation settings, we observe that \SCSE{} outperforms strong shared-block controls, with its clearest gains at deep extra-loop depths.

\section{Source-Centered State Evolution}
This paper studies the full-state, shared-block looped Transformer setting where the same block is applied repeatedly to the residual stream.  Let $x$ be a token sequence, $P_\phi$ a token-and-position embedding map, $R_\theta$ a shared recurrent core, and $C_\psi$ the output head.  A full-state looped Transformer computes $e=P_\phi(x)$, $h_0=H_0(e)$, $h_{t+1}=R_\theta(h_t,e)$ for $t=0,\ldots,T-1$, and next-token logits using root-mean-square layer normalization (RMSNorm) at readout, $C_\psi(\operatorname{RMSNorm}(h_T))$.  For a chosen loop budget, each recurrent step updates the full residual stream rather than using token-wise routing or a learned halting policy.  Here, the recurrently reused source is distinct from the fixed input-conditioned anchor.  Under additive injection, the source is reintroduced whenever the shared block is applied.  Thus, even at that anchor, the transition can produce the zero-deviation forcing bias $b_t(e)$ as a nonzero update whose finite-horizon task effect may be harmful, neutral, or beneficial.  Our goal is therefore to remove repeated source injection from the recurrent update while retaining input dependence through the fixed anchor and the initial deviation.

The proposed \SCSE{} parameterization selects anchor invariance as a direct architectural condition.  The anchor supplies a fixed origin for recurrent motion, the bias-free recurrent core acts on deviations from that origin, and the zero-deviation mask supplies an exact pointwise boundary condition.  Applying the recurrent core to $\Delta_t$ rather than $h_t$ defines a source-centered recurrent vector field in which the anchor remains separate from the evolving state.

Figure~\ref{fig:zdfb-intuition} illustrates the high-level design choice.  Here, $G_\theta$ is the shared Transformer update implemented with causal self-attention.  Additive source injection leaves a source-driven anchor response available to recurrent propagation, whereas the source-centered update selects anchor invariance and allocates computation to active nonzero deviations.

\begin{figure*}[t!]
	\centering
	\begin{overpic}[width=0.96\textwidth]{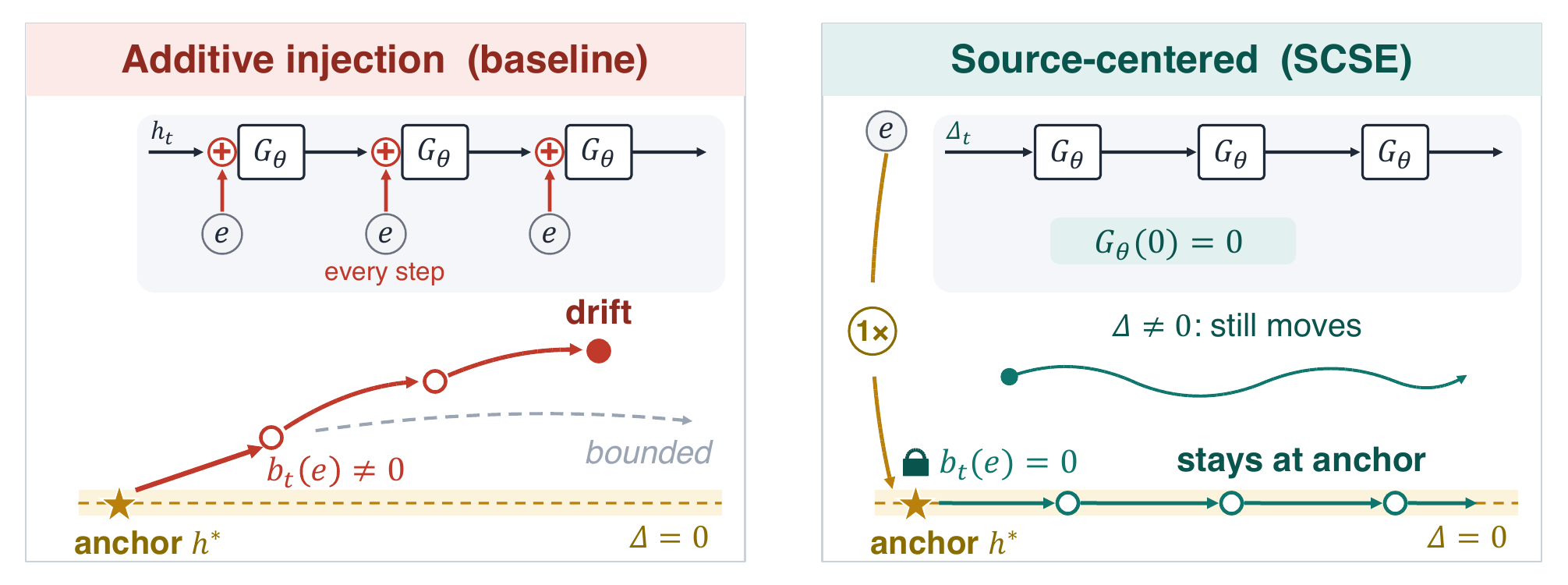}
	\end{overpic}
	\caption{Illustration of the zero-deviation forcing bias and its removal in \SCSE{}.  Repeated additive injection exposes a source-driven degree of freedom $b_t(e)$ at the anchor.  \SCSE{} uses $e$ once to set the anchor and a bias-free recurrent core for which $G_\theta(0)=0$ by construction.  Its zero-deviation mask additionally ensures that the input-conditioned anchor ($\Delta=0$) is a fixed point of the recurrent update at every step, while allowing active nonzero deviations to evolve.}
	\label{fig:zdfb-intuition}
\end{figure*}

\subsection{Source Control}
We compute an input-conditioned anchor
\begin{equation}
	h^\star=e+a_\omega(e),
\end{equation}
and evolve deviations $\Delta_t=h_t-h^\star$.  The anchor $h^\star$ is produced once per input sequence by applying the learned anchor projection module $a_\omega$ to the token-and-position representation $e$ and is then held fixed for all recurrent steps.  Thus, the loop evolves the anchor-relative displacement of the current hidden state from the input-conditioned reference instead of re-adding that reference at every step.  The learned anchor defines the origin for the zero-deviation condition.

Input dependence enters through both the anchor $h^\star(e)$ and the separately learned initial state $h_0=H_0(e)$, hence also through the initial deviation $\Delta_0=H_0(e)-h^\star(e)$.  Starting from $\Delta_0$, \SCSE{} evolves the deviation for $T$ recurrent steps to obtain $\Delta_T$, reconstructs $h_T=h^\star+\Delta_T$, and then applies the final readout.  Throughout this section, deviation refers to the anchor-relative displacement $\Delta_t=h_t-h^\star$.

\paragraph{Why Use Anchor Coordinates.} With repeated additive source injection, the shared block can produce a nonzero update even when the state is exactly at the anchor ($\Delta_t=0$).  As the block is applied repeatedly, this source-driven update may be attenuated, used productively, or accumulated.  \SCSE{} instead stores the input-conditioned anchor once and applies a zero-preserving recurrent update to the current deviation.  Thus, zero deviation remains fixed, while nonzero deviations can still evolve and support useful computation.  Specifically, this separation retains input dependence in $h^\star(e)$ and $\Delta_0$, while the recurrent core receives only $\Delta_t$, so additional loops evolve the current displacement without reapplying a fixed source.

Figure~\ref{fig:scse-intuition} summarizes the source-centered computation represented by the update equations.  Let $q_t^{(b)}$ and $\Delta_t^{(b)}$ denote the sequence-by-channel matrices for batch element $b$ at loop step $t$.  The implementation first forms a raw source-centered update $q_t$ and then applies a per-example zero-deviation mask $m_{b,t}$ to obtain $\bar q_t$:

\begin{figure}[t!]
	\centering
	\includegraphics[width=\columnwidth]{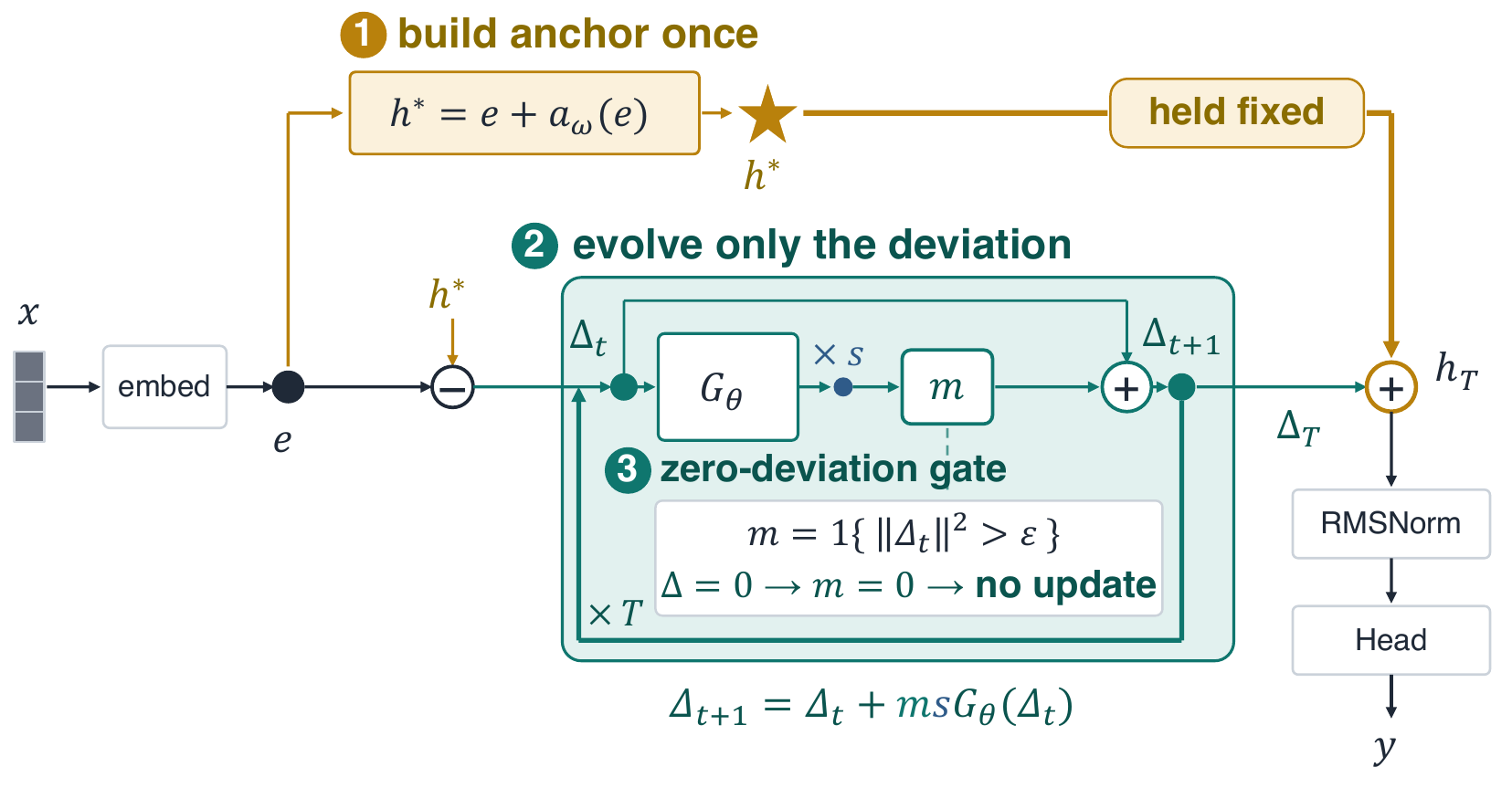}
	\caption{Illustration of \SCSE{}.  The proposed method builds the input-conditioned anchor $h^\star=e+a_\omega(e)$ once and holds the anchor fixed across all recurrent steps, evolves only the anchor-relative deviation $\Delta_t$, and reconstructs the final hidden state as $h^\star+\Delta_T$ before readout.  The zero-preserving recurrent block $G_\theta$ acts on the deviation with the residual step scale $s$, while the gate $m=\mathbf 1\{\norm{\Delta_t}_F^2>\epsilon\}$ enforces the exact pointwise zero-deviation boundary condition.}
	\label{fig:scse-intuition}
\end{figure}

\begin{align}
	q_t                & = sG_\theta(\Delta_t),           & D_{b,t}        & = \norm{\Delta_t^{(b)}}_F^2, \label{eq:source-start} \\
	m_{b,t}            & = \mathbf 1\{D_{b,t}>\epsilon\}, & \bar q_t^{(b)} & = m_{b,t}q_t^{(b)},                                  \\
	\Delta_{t+1}^{(b)} & =\Delta_t^{(b)}+\bar q_t^{(b)}.
	\label{eq:source}
\end{align}

When $\Delta_t^{(b)}=0$, we have $D_{b,t}=0$, which sets $m_{b,t}=0$ and $\bar q_t^{(b)}=0$.  Equation~\ref{eq:source} then gives $\Delta_{t+1}^{(b)}=\Delta_t^{(b)}=0$, equivalently $h_{t+1}^{(b)}=h_t^{(b)}=(h^\star)^{(b)}$, thereby yielding $\Tmap_t(0;e)=0$.  Thus, the designated anchor is a one-step fixed point by construction.

Here $s$ is the residual step scale.  In all reported experiments, the mask threshold is $\epsilon=10^{-8}$.  The per-example mask is active whenever $\norm{\Delta_t^{(b)}}_F^2>\epsilon$ and otherwise sets the recurrent increment to exactly zero throughout the small threshold region.  Let $\widetilde{\Tmap}_t$ denote the raw deviation map before masking and $\widetilde b_t(e)=\widetilde{\Tmap}_t(0;e)$ its raw anchor response.  For the \SCSE{} recurrence used in the main experiments, $\widetilde{\Tmap}_t(\Delta;e)=\Delta+sG_\theta(\Delta)$.  Because RMSNorm maps zero to zero and every attention and multilayer-perceptron projection in the core is bias-free, $G_\theta(0)=0$; hence $\widetilde b_t(e)=0$ even before masking.  Therefore, the source-centered, zero-preserving core is the primary reparameterization.  The mask supplies the exact pointwise boundary condition even if the underlying core is not zero-preserving.

Eqs.~\ref{eq:source-start} through~\ref{eq:source} specify one recurrent step.  In the end-to-end forward pass, the model computes $e=P_\phi(x)$, $h^\star=e+a_\omega(e)$, and $\Delta_0=H_0(e)-h^\star$ once, unrolls the shared update for $T$ steps with $h^\star$ held fixed, reconstructs $h_T=h^\star+\Delta_T$, and predicts next-token logits from $C_\psi(\operatorname{RMSNorm}(h_T))$.  Because $H_0$ and $h^\star$ are learned separately, the ordinary \SCSE{} trajectory generally starts with $\Delta_0\neq0$ and can compute immediately.  Starting exactly at $\Delta_0=0$ is instead a boundary test for the diagnostic below.  Choosing $h^\star=H_0(e)$ collapses the initial deviation and, as the anchor ablation in Table~\ref{tab:add-anchor-design} confirms, suppresses useful recurrent motion.

\subsection{Anchor-Consistency Condition and Diagnostic}
The preceding construction specifies how \SCSE{} computes and establishes the designated anchor as a one-step fixed point.  We now turn that condition into a diagnostic that can be evaluated on any trained shared-block model.  The Mechanistic Analysis section applies this diagnostic across trained model families, as reported in Table~\ref{tab:bias_decomp}.  Because $h^\star$ is fixed during the unroll, $\Delta_t=0$ places the hidden state exactly at the anchor.  Here, $\Tmap_t(\Delta;e)$ is the actual one-step deviation map before readout, with $h_t=h^\star+\Delta_t$ and $h_{t+1}=h^\star+\Tmap_t(\Delta_t;e)$.  The diagnostic $b_t(e)=\Tmap_t(0;e)$ measures the source-driven response at that anchor.  In \SCSE{}, this response vanishes by construction, as established after Eq.~\ref{eq:source}.  The quantity $b_t(e)$ tests only the realized one-step map at the anchor.  The diagnostic should not be conflated with the core's response before masking, $\widetilde b_t(e)$, and the pointwise value alone does not show how source forcing affects an ordinary trajectory over multiple steps.

On the active branch where $m_{b,t}=1$, \SCSE{} still differs from the looped Transformer baseline.  The baseline applies the shared transition to $h_t$, whereas \SCSE{} applies the shared transition to the anchor-relative deviation $\Delta_t$ and reconstructs states as $h^\star+\Delta_t$.  This reparameterization combines exact anchor invariance with state-dependent recurrent motion.

Additive-source recurrence can retain nonzero $b_t(e)=\Tmap_t(0;e)$ on the anchor-response test, whereas \SCSE{} sets both the pointwise term and, for its reported zero-preserving core, the raw anchor response to zero.  The appendix formalizes when the anchor response remains bounded under contraction, accumulates coherently, or increases task loss under adverse readout--loss alignment.

\section{Experiments}
\subsection{Setup}
The appendix provides detailed experimental specifications, including dataset splits, model dimensions, optimizer settings, benchmark protocols, floating-point operation (FLOP) accounting, matched token budgets, aggregation, and stochastic-replication counts.

The experiments evaluate \SCSE{} against a broad suite of shared-block controls and unshared references.  We complement these comparisons with anchor-design ablations, trained-model diagnostics of anchor response and source-channel steering, and direct intervention controls.  Additional evaluations in the appendix test the method across model scales, training budgets, adaptive loop budgets, context lengths, corpus shifts, completion benchmarks, and runtime settings.

During training, we sample the loop depth uniformly from $\{1,\ldots,8\}$.  At evaluation, we apply the same shared block at fixed depths $T\in\{4,8,12,16,24,32,48\}$, where available; separate adaptive-depth evaluations use an adaptive stopping rule over loop iterations.  We call $T\leq8$ the training loop-depth range and $T>8$ extra-loop evaluation.  This training protocol exposes all shared-block methods to multiple depths during optimization.

All main \SCSE{} experiments use the recurrence in Eq.~\ref{eq:source} with the residual step scale fixed at $s=0.50$.  The looped Transformer baseline follows the same loop-depth sampling, optimizer, sequence length, and evaluation schedule.  The main comparison suite includes tuned adapters, recurrent-step-conditioned adapters, parameter-matched unshared controls, and absolute unshared-depth references.  The tuned adapter uses the same learned anchor and initialization modules as \SCSE{} but keeps the additive-source update $h_{t+1}=h_t+s\Block_\theta(h_t+\alpha W_{\rm in}h^\star)$, with no recurrent-step term and no source-centered zero-deviation mask.  Here, $\alpha$ is that model's learned scalar source gain.  The reported tuned rows use the selected residual step scale shown in the row label, for example $s=0.35$ in the main comparisons.

\subsection{Shared-Block Control Comparisons}
\paragraph{Strong Shared-Block Transformer Baselines.} Appendix Table~\ref{tab:strong_followup} reports the stronger shared-block comparison on WikiText-103 at the 22M scale.  The capacity-matched and tuned adapters are substantially stronger than the looped Transformer baseline, yet \SCSE{} attains lower mean PPL than all three baselines at every loop depth reported in Appendix Table~\ref{tab:strong_followup}.  These comparisons support the benefit of the full source-centered reparameterization.  Appendix Table~\ref{tab:add-causal-forcing} separately probes the zero-deviation mask and subtraction of the anchor response throughout the active update.  Removing the mask from the source-conditioned anchor-coordinate (SC-Cond) reference leaves PPL nearly unchanged, whereas the separately trained two-body subtraction controls improve deep extra-loop PPL.

\paragraph{Recurrent-Step-Conditioned Controls.} To test whether explicit loop-step information can account for \SCSE{}'s gains, we construct a recurrent-step-conditioned control, drawing on a mechanism used by the Universal Transformer and many modern looped-depth designs \citep{dehghani2019universal,fan2025looped,geiping2025huginn,zhu2025ouro,jeddi2026loopformer,chen2026thinking}.  Public looped-depth systems pair recurrent-step conditioning with other design choices involving scale, data, loop placement, training objective, adaptive computation, stable injection, and learned depth allocation.  Concretely, the recurrent-step-conditioned control adds a learned projection of a sinusoidal loop-step embedding to the shared block's input at each recurrence.  All other training settings are unchanged.  The added parameter count is small relative to the model scale.
\begin{equation}
	h_{t+1}
	=
	h_t + s\Block_\theta(h_t+\alpha W_{\rm in}h^\star+\tau P_{\rm step}\gamma_t),
\end{equation}
where $\gamma_t$ is the sinusoidal loop-step embedding used at loop step $t$, $P_{\rm step}$ is a learned recurrent-step projection, distinct from the embedding map $P_\phi$, and $\tau$ is tuned on WikiText-103 22M.  The best setting in the recurrent-step-conditioned control is small, with $\tau=0.015$.

Table~\ref{tab:stepcond_wt103} shows the resulting comparison.  At 22M, the recurrent-step-conditioned adapter is essentially tied with the tuned adapter at $T=8$, with $156.7\pm0.5$ PPL compared with $156.4\pm0.6$, and trails by $1.2$ PPL at $T=24$ and $2.1$ PPL at $T=48$.  \SCSE{} keeps the best shared-block quality frontier at $T=8$, $T=24$, and $T=48$.

At 50M, the recurrent-step-conditioned adapter becomes a genuinely strong recurrent baseline.  The adapter improves the looped Transformer baseline from $151.1$ to $125.7$ PPL at $T=8$ and from $178.9$ to $160.1$ at $T=48$.  At this scale, \SCSE{} attains the lowest PPL among the compared shared-block methods at every measured loop depth, reaching $123.1$ at $T=8$, $135.5$ at $T=24$, and $156.4$ at $T=48$.

At 95.6M, the same control remains competitive and informative.  The control reaches $100.8\pm0.5$ PPL at $T=8$, $111.9\pm0.3$ at $T=24$, and $130.6\pm1.0$ at $T=48$, while \SCSE{} gives the best PPL at every measured depth.

Appendix Tables~\ref{tab:heldout_transfer} and~\ref{tab:lambada} further evaluate the same recurrent-step-conditioned control on held-out web-text transfer and the Language Modeling Broadened to Account for Discourse Aspects (LAMBADA) completion benchmark.  The recurrent-step-conditioned adapter becomes competitive at larger scales and even achieves the lowest $T=8$ PPL in the 50M held-out OpenWebText setting, but the adapter remains behind \SCSE{} in the reported deep extra-loop and completion comparisons.  Thus, the recurrent-step-conditioned adapter is a strong baseline, but recurrent-step conditioning alone does not account for the source-centered quality advantage.

\begin{table}[t!]
	\centering
	\scriptsize
	\resizebox{\columnwidth}{!}{
		\begin{tabular}{lrrr}
			\toprule
			                                               & \multicolumn{3}{c}{PPL}                                                   \\
			\cmidrule(lr){2-4}
			Method                                         & $T=8$                   & $T=24$                 & $T=48$                 \\
			\midrule
			\multicolumn{4}{l}{22M}                                                                                                    \\
			Tuned adapter, $s=0.35^{\dagger}$              & $156.4\pm0.6$           & $174.4\pm0.6$          & $204.1\pm1.4$          \\
			Recurrent-step-conditioned adapter$^{\dagger}$ & $156.7\pm0.5$           & $175.6\pm0.7$          & $206.1\pm1.4$          \\
			\SCSE{}, $s=0.50$                              & $\mathbf{155.1\pm0.6}$  & $\mathbf{171.1\pm0.8}$ & $\mathbf{200.1\pm1.3}$ \\
			\addlinespace[2pt]
			\midrule
			\multicolumn{4}{l}{50M}                                                                                                    \\
			Looped Transformer baseline                    & $151.1\pm0.9$           & $162.5\pm1.0$          & $178.9\pm1.3$          \\
			Recurrent-step-conditioned adapter             & $125.7\pm0.4$           & $139.2\pm0.4$          & $160.1\pm0.2$          \\
			\SCSE{}                                        & $\mathbf{123.1\pm0.2}$  & $\mathbf{135.5\pm0.4}$ & $\mathbf{156.4\pm1.1}$ \\
			\addlinespace[2pt]
			\midrule
			\multicolumn{4}{l}{95.6M}                                                                                                  \\
			Looped Transformer baseline                    & $117.1\pm0.5$           & $128.0\pm0.7$          & $145.1\pm1.2$          \\
			Tuned adapter, $s=0.35$                        & $98.6\pm0.5$            & $109.8\pm0.9$          & $128.7\pm0.9$          \\
			Recurrent-step-conditioned adapter             & $100.8\pm0.5$           & $111.9\pm0.3$          & $130.6\pm1.0$          \\
			\SCSE{}                                        & $\mathbf{96.9\pm0.3}$   & $\mathbf{107.6\pm0.6}$ & $\mathbf{125.9\pm0.9}$ \\
			\bottomrule
		\end{tabular}
	}
	\caption{Comparison of \SCSE{} with shared-block controls on WikiText-103 under matched training protocols.  The recurrent-step-conditioned control adds a small projected sinusoidal loop-step embedding to the tuned-adapter protocol.  Rows marked $^\dagger$ are explicitly tuned 22M hyperparameter variants.  \SCSE{} remains the best shared-block quality frontier across the reported loop depths.}
	\label{tab:stepcond_wt103}
\end{table}

\subsection{Depth Response Across Loop Depths}

Figure~\ref{fig:wt103-depth-response} shows the 95.6M fixed-depth response across the measured loop-depth range.  The comparison evaluates one set of trained weights across evaluation depths rather than training a separate model per depth.  Within the training loop-depth range, the expected benefit of additional loops in looped Transformers \citep{zhu2025ouro,prairie2026parcae} holds: PPL improves from $T=4$ to $T=8$ for every method with both depths reported, and each method's minimum over the measured depths occurs at $T=8$.  Beyond $T=8$, WikiText-103 test PPL increases at each subsequently reported depth for every method, indicating worse in-domain performance.  \SCSE{} retains the lowest PPL among the compared shared-block methods at each depth.  From $T=8$ to $T=48$, PPL increases by $28.0$ for the looped Transformer baseline, $30.1$ for the tuned adapter, $29.8$ for the recurrent-step-conditioned adapter, and $29.0$ for \SCSE{}.  On held-out OpenWebText, extra loops beyond the training loop-depth range instead improve \SCSE{} at every scale while the additive-source controls degrade, as shown in Appendix Table~\ref{tab:heldout_transfer}.

\begin{figure}[t!]
	\centering
	\includegraphics[width=\columnwidth]{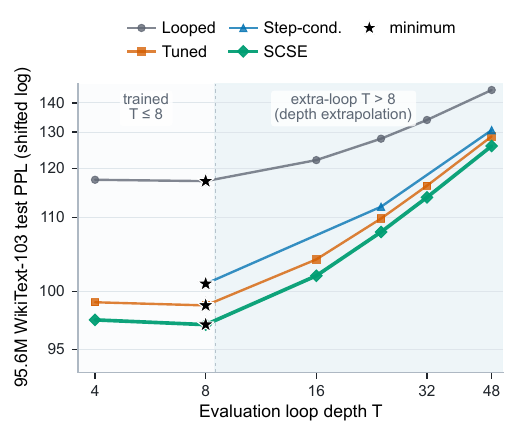}
	\caption{Mean test PPL versus evaluation loop depth for 95.6M WikiText-103 models.  For each trained run, the same weights are evaluated at all loop depths.  For every method reported at both depths, PPL improves from $T=4$ to $T=8$ within the training loop-depth range; stars mark the per-method minima, and the shaded region $T>8$ denotes depth extrapolation.  Beyond $T=8$, test PPL increases at each subsequently reported depth for every method, while \SCSE{} retains the lowest PPL among the compared shared-block methods at each depth.  Loop depth is plotted on a base-2 logarithmic axis; the PPL axis uses the shifted-log transform $\log(\mathrm{PPL}-84.85)$, with tick labels shown in the original PPL units.}
	\label{fig:wt103-depth-response}
\end{figure}

\subsection{Anchor Design Ablation}
\label{sec:anchor-design-ablation}

We vary only the anchor definition within \SCSE{}; all variants use the recurrence in Eq.~\ref{eq:source} and the same training protocol.  The learned-anchor row, $h^\star=e+a_\omega(e)$, corresponds to the main \SCSE{} model.  The ablation compares the learned anchor with the raw embedding anchor $h^\star=e$, an initial-state anchor $h^\star=H_0(e)$, and a frozen random projection with the same module shape as $a_\omega$.

\begin{table}[t!]
	\centering
	\scriptsize
	\resizebox{\columnwidth}{!}{
		\begin{tabular}{lrrr}
			\toprule
			                                & \multicolumn{3}{c}{Test PPL}                                                       \\
			\cmidrule(lr){2-4}
			Anchor                          & $T=8$                        & $T=24$                   & $T=48$                   \\
			\midrule
			Learned $h^\star=e+a_\omega(e)$ & $\mathbf{155.14\pm0.61}$     & $\mathbf{171.12\pm0.76}$ & $\mathbf{200.10\pm1.31}$ \\
			Embedding $e$                   & $160.55\pm1.44$              & $177.32\pm0.99$          & $206.16\pm0.99$          \\
			Initial $H_0(e)$                & $294.37\pm1.24$              & $294.37\pm1.24$          & $294.37\pm1.24$          \\
			Frozen random projection        & $159.14\pm0.78$              & $175.74\pm0.89$          & $203.78\pm0.94$          \\
			\bottomrule
		\end{tabular}
	}
	\caption{Anchor design ablation on WikiText-103.  The learned anchor has the lowest PPL among the tested source-centered anchors.  Setting the anchor equal to the initial state collapses the initial deviation to zero and prevents useful recurrent motion under the zero-deviation constraint, yielding the flat depth curve.}
	\label{tab:add-anchor-design}
\end{table}

Table~\ref{tab:add-anchor-design} shows that the learned anchor contributes beyond parameter capacity.  The raw embedding anchor and frozen random anchor both have higher PPL than the learned anchor at all reported depths.  Capacity-matched, tuned, and recurrent-step-conditioned adapters instantiate the same learned anchor and initialization modules but use additive recurrent updates, isolating the effect of source-centered recurrence.  Setting the anchor to $h^\star=H_0(e)$ collapses the initial deviation to zero: $\Delta_0=H_0(e)-h^\star=0$.  The zero-deviation constraint then keeps $\Delta_t=0$ at every loop step, so the recurrent branch makes no update and PPL is independent of loop depth.  The result favors a fixed anchor that remains separate from the initial recurrent state, while leaving the design of more effective learned anchor modules open.

\section{Mechanistic Analysis}
\label{sec:analysis}

Figure~\ref{fig:wt103-depth-response} shows the task-level depth-response pattern; however, task-level PPL alone does not identify the underlying mechanism.  We therefore address two specific mechanism-level questions: (i) whether trained maps exhibit the predicted pointwise anchor response, and (ii) whether source-channel interventions can steer recurrent motion in an anchor-started rollout.

\begin{figure*}[t!]
	\centering
	\begin{overpic}[width=0.99\textwidth]{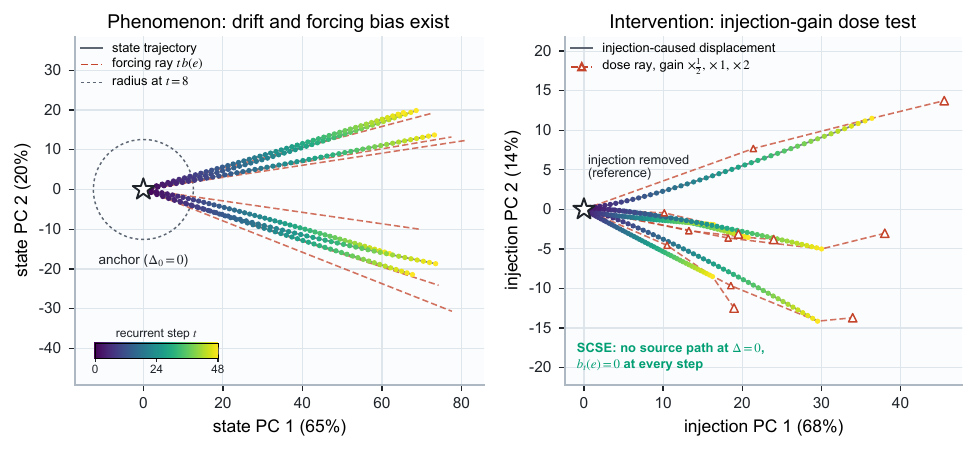}
		\put(1,44){\small (a)}
		\put(51,44){\small (b)}
	\end{overpic}
	\caption{Anchor-start diagnostic of forcing-aligned motion using trained weights from one WikiText-2 model; all rollouts set $\Delta_0=0$.  (a) In the additive-injection model, a nonzero anchor response initiates motion away from the anchor, and the ensuing trajectories remain aligned with the input-specific forcing rays $t b(e)$, visualizing forcing-aligned anchor drift.  The dotted circle marks the mean anchor distance at the largest trained depth ($t=8$) and is not a stability boundary.  (b) Varying only the evaluation-time injection gain yields collinear injection-relative displacements whose magnitudes scale with the gain, providing dose-response evidence that the source channel steers recurrent motion.  Triangles mark the half-gain and double-gain endpoints.  The panels use separate principal-component projections, and the green label marks the architectural contrast $b_t(e)=0$.}
	\label{fig:zdfb-empirical}
\end{figure*}

\begin{table}[t!]
	\centering
	\scriptsize
	\resizebox{\columnwidth}{!}{
		\begin{tabular}{lrrr}
			\toprule
			Method                             & $R_0$ & $R_{47}$ & $E_{\Delta,47}$ \\
			\midrule
			\multicolumn{4}{l}{22M}                                                 \\
			Looped Transformer baseline        & 1.000 & 4.351    & 44.7            \\
			Tuned adapter, $s=0.35$            & 0.380 & 1.619    & 9.8             \\
			Recurrent-step-conditioned adapter & 0.329 & 1.276    & 9.7             \\
			\SCSE{}                            & 0.000 & 0.000    & 18.1            \\
			\midrule
			\multicolumn{4}{l}{50M}                                                 \\
			Looped Transformer baseline        & 1.000 & 4.296    & 477.2           \\
			Recurrent-step-conditioned adapter & 0.264 & 1.053    & 64.7            \\
			\SCSE{}                            & 0.000 & 0.000    & 109.9           \\
			\midrule
			\multicolumn{4}{l}{95.6M}                                               \\
			Looped Transformer baseline        & 1.000 & 5.436    & 2129.4          \\
			Tuned adapter, $s=0.35$            & 0.239 & 1.331    & 321.5           \\
			Recurrent-step-conditioned adapter & 0.151 & 0.894    & 332.1           \\
			\SCSE{}                            & 0.000 & 0.000    & 603.8           \\
			\bottomrule
		\end{tabular}
	}
	\caption{Anchor-response forcing-bias diagnostic using trained WikiText-103 weights and recurrent unrolling through $T=48$.  $R_t$ is the pointwise bias energy normalized by realized update energy, and $E_{\Delta,t}:=N^{-1}\norm{\Delta_t}_F^2$ is the mean-squared anchor energy.  Within each model family's anchor coordinates, \SCSE{} has zero pointwise response at the reported steps, whereas the additive-source models retain nonzero late-step responses.  The nonzero $E_{\Delta,47}$ values serve only as a non-collapse check, showing that ordinary \SCSE{} trajectories remain far outside the zero-deviation threshold region; $E_{\Delta}$ magnitudes depend on each family's anchor and are not compared across families.}
	\label{tab:bias_decomp}
\end{table}

\paragraph{Pointwise Anchor-Response Diagnostic.} To test the predicted pointwise anchor-response contrast, we measure one-step responses at zero deviation.  Each model family uses its own fixed input-conditioned anchor: $h^\star=e+a_\omega(e)$ for anchored adapters and source-centered variants, and $h^\star=e$ for the additive-injection baseline.  Accordingly, we apply the pointwise anchor-response diagnostic separately within each trained model family.

Recall that $\Tmap_t(\Delta;e)$ denotes the one-step deviation map in anchor coordinates before readout, with $\Delta_{t+1}=\Tmap_t(\Delta_t;e)$.  We decompose
\begin{equation}
	\Tmap_t(\Delta;e)=\Delta+b_t(e)+a_t(\Delta;e),
	\label{eq:forcing-decomposition}
\end{equation}
where
\begin{equation}
	\begin{aligned}
		b_t(e)        & :=\Tmap_t(0;e),                          \\
		a_t(\Delta;e) & :=\Tmap_t(\Delta;e)-\Tmap_t(0;e)-\Delta.
	\end{aligned}
\end{equation}
The term $b_t(e)$ is the pointwise zero-deviation forcing bias, abbreviated as the forcing bias in experimental labels.

Let $N$ be the number of scalar entries in a measured batch tensor.  The implemented forcing-bias energy ratio is
\begin{equation}
	R_t(e)
	:=
	\frac{N^{-1}\norm{b_t(e)}_F^2}
	{\max\{N^{-1}\norm{\Delta_{t+1}-\Delta_t}_F^2,10^{-12}\}}.
	\label{eq:forcing-bias-ratio}
\end{equation}
Thus, $R_t(e)$ measures pointwise anchor-response energy relative to realized recurrent-update energy: zero indicates no measured pointwise response and, when the numerical floor is inactive, one indicates equal energies.  Values above one can reflect cancellation between pointwise and state-dependent terms and are not causal contribution percentages.

Table~\ref{tab:bias_decomp} verifies the architectural contrast across scales.  These results answer the first question: additive-source models retain measurable pointwise responses across scales ($R_{47}=0.894$--$5.436$), whereas \SCSE{} removes both the pointwise and raw pre-mask anchor responses while preserving nonzero anchor-relative motion.
\paragraph{Anchor-Start Source-Steering Diagnostic.} Figure~\ref{fig:zdfb-empirical} complements these aggregate measurements with a geometric case study using trained weights from one WikiText-2 model.  The study uses the 22M shared-block 1200-step protocol with one stochastic replication and six held-out test inputs.  The diagnostic rollouts deliberately start from the anchor rather than the model's ordinary initialization.  Additive-injection trajectories follow input-specific forcing rays with trajectory--ray cosines between $0.78$ and $1.00$.  Rescaling only the evaluation-time injection gain produces dose-aligned displacement paths relative to the injection-removed rollout.  These paths grow monotonically and span $25\%$ to $51\%$ of the full sequence-mean anchor push at $t=48$.  Their root-mean-square magnitude is about $36\%$ of the anchor push.  Their directions are nearly orthogonal to the realized drift.

More specifically, Figure~\ref{fig:zdfb-empirical}(a) evaluates the full learned block response at the anchor, whereas Figure~\ref{fig:zdfb-empirical}(b) provides intervention-based evidence of source-channel steering.  Within this anchor-start diagnostic, the injection-gain intervention answers the second question affirmatively: changing only the evaluation-time injection gain steers recurrent motion in a dose-aligned manner.  The appendix separately provides the exact $b_t$-subtraction analysis and direct forcing-bias controls relevant to broader drift and PPL attribution.

\section{Conclusion}
Repeated source injection leaves a source-driven anchor response as a degree of freedom in shared recurrent dynamics.  The contribution of the anchor response is shaped by recurrent propagation and readout--loss alignment, as characterized by the exact bias-subtraction counterfactual.  \SCSE{} removes the source-driven anchor-response degree of freedom by imposing the anchor-consistency condition.  A learned anchor and an initial deviation retain input dependence, while a zero-preserving deviation core allocates recurrent computation to nonzero deviations.  The resulting architectural separation is a desirable property of shared recurrence: the anchor remains a fixed point, yet nonzero deviations can still support input-dependent computation.  Across fixed-depth, extra-loop, adaptive-depth, transfer, completion, context-length, runtime, and scale comparisons, \SCSE{} improves the shared-block recurrent quality frontier.  Ablation studies identify the learned anchor and the anchor-coordinate deviation recurrence as the main contributors, and trained-model diagnostics connect the anchor response to observed recurrent motion.  These results support source-centered coordinates as an effective inductive bias for reusable recurrent depth.

\bibliography{aaai2027}

\clearpage
\appendix

\section{Appendix}
\label{sec:appendix}

\subsection{List of Notation}
\begin{table}[t!]
	\centering
	\scriptsize
	\resizebox{\columnwidth}{!}{
		\begin{tabular}{ll}
			\toprule
			Symbol                                  & Meaning                                                                           \\
			\midrule
			$x$                                     & Token sequence                                                                    \\
			$d$                                     & Hidden width                                                                      \\
			$d_{\rm ff},L$                          & Feed-forward width and sequence length                                            \\
			$P_\phi$                                & Token-and-position embedding map                                                  \\
			$e=P_\phi(x)$                           & Input representation                                                              \\
			$H_0$                                   & Initial-state map, $h_0=H_0(e)$                                                   \\
			$h_t$                                   & Hidden state at recurrent step $t$                                                \\
			$T,t$                                   & Loop depth and recurrent-step index                                               \\
			$N_{\rm body}(T),N_{\rm layers}$        & Active body applications and unshared layers                                      \\
			$c_{\rm body}$                          & Body evaluations per logical recurrent step                                       \\
			$R_\theta$                              & Recurrent transition                                                              \\
			$G_\theta,\Block_\theta$                & Shared update maps                                                                \\
			$C_\psi$                                & Output head and readout map                                                       \\
			$h^\star=e+a_\omega(e)$                 & Fixed input-conditioned anchor                                                    \\
			$\Delta_t=h_t-h^\star$                  & Anchor-relative deviation                                                         \\
			$q_t,\bar q_t$                          & Raw and masked source-centered updates                                            \\
			$s,\lambda_{\rm leak},\kappa$           & Step scale, leak, and diagnostic conditioning                                     \\
			$W_c,W_{\rm in}$                        & Conditioning and additive-source projections                                      \\
			$A_t^{\rm sec},\Phi_E(t,k)$             & Exact secant maps and transition products                                         \\
			$\rho$                                  & Real eigenvalue associated with a left eigenvector of the constant secant map $A$ \\
			$\bar\Delta_t,E_t,p_{T,k}$              & Bias-subtracted path and propagated response                                      \\
			$\widehat E_T$                          & Unit direction of the final forcing response                                      \\
			$\Tmap_t(\Delta;e)$                     & Actual anchor-coordinate one-step map                                             \\
			$\widetilde{\Tmap}_t,\widetilde b_t(e)$ & Raw pre-mask map and anchor response                                              \\
			$b_t(e),R_t(e)$                         & Pointwise forcing bias and reported energy ratio                                  \\
			$a_t(\Delta;e)$                         & State-dependent residual term                                                     \\
			$\alpha,\tau_t$                         & Learned additive-source gain and optional step conditioning                       \\
			$|\mathcal V|$                          & Vocabulary size                                                                   \\
			$z,y$                                   & Next-token logits and target                                                      \\
			$\mathcal J_e$                          & Task loss in anchor coordinates                                                   \\
			$\norm{\cdot}$                          & Vector or matrix norm                                                             \\
			\bottomrule
		\end{tabular}
	}
	\caption{List of notation.}
	\label{tab:notation}
\end{table}

\subsection{When Zero-Deviation Forcing Accumulates}
\label{app:forcing-accumulation}
The definition $b_t(e)=\Tmap_t(0;e)$ establishes whether the chosen anchor is a one-step fixed point, but the pointwise definition does not establish whether that pointwise term changes an ordinary off-anchor trajectory.  We therefore analyze an explicit counterfactual that subtracts the same anchor response throughout the learned map.  This intervention is a mathematical comparison, not an additional empirical result.  For a discontinuous masked map, the result can instead be applied to a continuously differentiable raw or active map or to a separately defined smooth subtractive construction.  In that case, throughout this subsection $\Tmap_t$ and $b_t$ denote the selected differentiable map and its own anchor response.  In particular, selecting the raw map means using $\widetilde{\Tmap}_t$ and $\widetilde b_t$.  The result cannot be inferred merely from changing the masked branch within the zero-deviation threshold region.

Fix an input $e$ and an interval $t=t_0,\ldots,T-1$.  Using the exact decomposition in Eq.~\ref{eq:forcing-decomposition}, define the original and bias-subtracted trajectories from the same initial state by
\begin{align}
	\Delta_{t+1}
	 & =\Delta_t+b_t(e)+a_t(\Delta_t;e),                                      \notag \\
	\bar\Delta_{t+1}
	 & =\Tmap_t(\bar\Delta_t;e)-b_t(e)
	=\bar\Delta_t+a_t(\bar\Delta_t;e),                         \notag                \\
	\bar\Delta_{t_0}
	 & =\Delta_{t_0}.
	\label{eq:bias-subtracted-trajectories}
\end{align}
Throughout this subsection, each sequence-by-channel deviation is identified with its vectorization in $\R^{Ld}$.  Accordingly, $\langle\cdot,\cdot\rangle$ and $\norm{\cdot}_2$ denote the Euclidean inner product and norm.  Before vectorization, these are the Frobenius inner product and norm.  The Jacobians below act on vectorized states.  Let $E_t=\Delta_t-\bar\Delta_t$ denote the counterfactual forcing response.  Assume $a_t(\cdot;e)$ is continuously differentiable on a neighborhood of the line segment joining $\bar\Delta_t$ and $\Delta_t$, and define the exact secant propagation map
\begin{equation}
	\begin{aligned}
		A_t^{\rm sec}
		 & :=I+\int_0^1
		J_\Delta a_t(\bar\Delta_t+\xi E_t;e)d\xi \\
		 & =\int_0^1
		J_\Delta\Tmap_t(\bar\Delta_t+\xi E_t;e)d\xi.
	\end{aligned}
	\label{eq:secant-propagation}
\end{equation}
For $t>k$, define $\Phi_E(t,k)=A_{t-1}^{\rm sec}\cdots A_k^{\rm sec}$ and $\Phi_E(k,k)=I$.

\begin{lemma}[Exact one-step response recursion]
	\label{lem:onestep-forcing}
	The nonlinear trajectories in Eq.~\ref{eq:bias-subtracted-trajectories} satisfy
	\begin{equation}
		E_{t+1}
		=A_t^{\rm sec}E_t+b_t(e).
		\label{eq:exact-response-step}
	\end{equation}
\end{lemma}

\begin{theorem}[Exact counterfactual forcing propagation]
	\label{thm:counterfactual-forcing}
	For the nonlinear trajectories in Eq.~\ref{eq:bias-subtracted-trajectories}, the counterfactual forcing response satisfies
	\begin{equation}
		E_T
		=\sum_{k=t_0}^{T-1}\Phi_E(T,k+1)b_k(e).
		\label{eq:exact-finite-horizon-forcing}
	\end{equation}
\end{theorem}

\begin{remark}[Exactness and choice of interval]
	By Theorem~\ref{thm:counterfactual-forcing}, the response to subtracting the pointwise bias is an exact finite-horizon sum for the nonlinear learned map, with propagation evaluated along the paired trajectories.  The start $t_0$ is arbitrary and can be chosen to coincide with the boundary of the training loop-depth range.
\end{remark}

For compactness, write $p_{T,k}=\Phi_E(T,k+1)b_k(e)$ and $n=T-t_0$.

\begin{corollary}[Contractive forcing response]
	\label{cor:contractive-forcing}
	Under Theorem~\ref{thm:counterfactual-forcing}, for any vector norm and its induced matrix norm, if there are constants $0\leq q<1$ and $B\geq0$ such that $\norm{A_t^{\rm sec}}\leq q$ and $\norm{b_t(e)}\leq B$ throughout the interval, then
	\begin{equation}
		\norm{E_T}
		\leq B\frac{1-q^n}{1-q}
		\leq\frac{B}{1-q}.
	\end{equation}
\end{corollary}

\begin{corollary}[Coherent forcing accumulation]
	\label{cor:coherent-forcing}
	Under Theorem~\ref{thm:counterfactual-forcing} and the Euclidean norm, if there is a unit vector $v$ such that $\langle v,p_{T,k}\rangle\geq\beta_k\geq0$ for every $k$, then
	\begin{equation}
		\norm{E_T}_2\geq\langle v,E_T\rangle
		\geq\sum_{k=t_0}^{T-1}\beta_k.
	\end{equation}
	In particular, if $\beta_k\geq\beta>0$ for every $k$ in the interval, then $\norm{E_T}_2\geq n\beta$ at this horizon.
\end{corollary}

\begin{corollary}[Constant-secant spectral response]
	\label{cor:spectral-forcing}
	Under Theorem~\ref{thm:counterfactual-forcing}, suppose $A_t^{\rm sec}=A$ and $b_t(e)=b(e)$ throughout the chosen interval.  If a Euclidean unit vector $v$ satisfies $A^\top v=\rho v$ for a real $\rho$, then
	\begin{equation}
		\langle v,E_T\rangle
		=\langle v,b(e)\rangle\sum_{\ell=0}^{n-1}\rho^\ell.
	\end{equation}
\end{corollary}

Time homogeneity of the learned map alone does not imply $A_t^{\rm sec}=A$.  Even when $\Tmap_t$ has no explicit step dependence, its exact secant map can vary because the paired trajectory segments vary with $t$.

To state the additional task-level condition, let
\begin{equation}
	\mathcal J_e(\Delta)
	:=\ell(C_\psi(\operatorname{RMSNorm}(h^\star+\Delta)),y)
\end{equation}
denote the token-level or sequence-aggregated negative log-likelihood for the target $y$.

\begin{proposition}[Sufficient task-loss bound under adverse readout alignment]
	\label{prop:task-loss-forcing}
	Suppose $E_T\neq0$, set $\widehat E_T=E_T/\norm{E_T}_2$, and assume $\mathcal J_e$ is continuously differentiable on a neighborhood of the segment from $\bar\Delta_T$ to $\Delta_T$.  If, for some $\mu>0$ and every $\xi\in[0,1]$,
	\begin{equation}
		\langle
		\nabla\mathcal J_e(\bar\Delta_T+\xi E_T),\widehat E_T
		\rangle\geq\mu,
	\end{equation}
	then
	\begin{equation}
		\mathcal J_e(\Delta_T)-\mathcal J_e(\bar\Delta_T)
		\geq\mu\norm{E_T}_2.
	\end{equation}
\end{proposition}

Under the uniform coherent condition $\beta_k\geq\beta>0$ in Corollary~\ref{cor:coherent-forcing}, $\norm{E_T}_2\geq n\beta$, so the loss gap in Proposition~\ref{prop:task-loss-forcing} is at least $\mu n\beta$ at that horizon.  The uniform lower bound $\mu$ is sufficient, not necessary.  For $E_T\neq0$, the exact integral identity in the proof shows that a positive integrated directional derivative is necessary and sufficient for a positive loss gap; a zero integral makes the response invisible to this loss, and a negative integral makes the forcing beneficial relative to the bias-subtracted path.

These statements separate pointwise presence, propagated state response, horizon scaling, and task harm.  A nonzero $b_t(e)$ can produce a bounded input-dependent offset under contraction, be neutralized by cancellation, or support useful early recurrent computation; repeated injection can therefore be compatible with and sometimes helpful for fixed-point computation \citep{anil2022path,blayney2026mechanistic}.  A bounded nonzero response can still raise or lower task loss according to its integrated readout--loss alignment.  Coherent accumulation together with uniform adverse alignment yields the stronger depth-linear lower bound $\mu n\beta$, but neither property follows from $b_t(e)\neq0$ alone.  A scale-invariant readout can hide some forms of state drift \citep{sharma2026readout}, while leaving the trained operating region is not by itself sufficient for degradation.  Conversely, setting the pointwise $b_t(e)=0$ does not constrain the state-dependent term $a_t$, the Jacobian away from the anchor, or global stability.

\subsection{Proofs}
\label{app:proofs}

\subsubsection{Counterfactual Forcing Propagation}
\paragraph{Lemma~\ref{lem:onestep-forcing}.}
\begin{proof}
	Subtracting the two recurrences in Eq.~\ref{eq:bias-subtracted-trajectories} gives
	\begin{align*}
		E_{t+1}
		=E_t+b_t(e)+a_t(\Delta_t;e)-a_t(\bar\Delta_t;e).
	\end{align*}
	Because $\Delta_t=\bar\Delta_t+E_t$, the vector-valued fundamental theorem of calculus gives
	\begin{align*}
		a_t(\Delta_t;e)-a_t(\bar\Delta_t;e)
		 & =\int_0^1J_\Delta a_t(\bar\Delta_t+\xi E_t;e)E_t d\xi \\
		 & =(A_t^{\rm sec}-I)E_t.
	\end{align*}
	The preceding calculation proves Eq.~\ref{eq:exact-response-step}.
\end{proof}

\paragraph{Theorem~\ref{thm:counterfactual-forcing}.}
\begin{proof}
	Since $E_{t_0}=0$, repeated substitution of Eq.~\ref{eq:exact-response-step} in Lemma~\ref{lem:onestep-forcing} yields
	\begin{align*}
		E_T
		=\sum_{k=t_0}^{T-1}\Phi_E(T,k+1)b_k(e),
	\end{align*}
	which is Eq.~\ref{eq:exact-finite-horizon-forcing}.
\end{proof}

\paragraph{Corollary~\ref{cor:contractive-forcing}.}
\begin{proof}
	For the contractive case, compatibility of a vector norm with its induced matrix norm and submultiplicativity give
	\begin{align*}
		\norm{\Phi_E(T,k+1)b_k(e)}
		\leq q^{T-k-1}B.
	\end{align*}
	Summing the geometric series proves the claim.
\end{proof}

\paragraph{Corollary~\ref{cor:coherent-forcing}.}
\begin{proof}
	Cauchy--Schwarz, Theorem~\ref{thm:counterfactual-forcing}, and the assumed projections give
	\begin{align*}
		\norm{E_T}_2
		\geq\langle v,E_T\rangle
		=\sum_{k=t_0}^{T-1}\langle v,p_{T,k}\rangle
		\geq\sum_{k=t_0}^{T-1}\beta_k.
	\end{align*}
	When $\beta_k\geq\beta>0$ for every $k$, this sum has $n=T-t_0$ terms, each at least $\beta$, so $\norm{E_T}_2\geq n\beta$.
\end{proof}

\paragraph{Corollary~\ref{cor:spectral-forcing}.}
\begin{proof}
	In the homogeneous case,
	\begin{align*}
		E_T=\sum_{\ell=0}^{n-1}A^\ell b(e).
	\end{align*}
	The left-eigenvector condition implies $(A^\ell)^\top v=\rho^\ell v$, so taking the inner product with $v$ gives the spectral formula.
\end{proof}

\paragraph{Proposition~\ref{prop:task-loss-forcing}.}
\begin{proof}
	Write $\Delta_\xi^{\rm seg}=\bar\Delta_T+\xi E_T$.  The scalar fundamental theorem of calculus along the segment joining the two final states gives
	\begin{align*}
		\mathcal J_e(\Delta_T)-\mathcal J_e(\bar\Delta_T)
		 & =\int_0^1
		\langle\nabla\mathcal J_e(\Delta_\xi^{\rm seg}),E_T\rangle d\xi          \\
		 & =\norm{E_T}_2\int_0^1
		\langle\nabla\mathcal J_e(\Delta_\xi^{\rm seg}),\widehat E_T\rangle d\xi \\
		 & \geq\mu\norm{E_T}_2.
	\end{align*}
\end{proof}

\subsection{Reproducibility Details}
The reproducibility details summarize how stochastic replications, evaluation identities, aggregation, and verification checks are defined.  This organization keeps the numerical claims tied to semantic evaluation settings rather than local storage paths.

\subsubsection{Experimental Setup Details}
\label{app:experimental-setup-details}
We train decoder-only LMs from scratch on WikiText-2 and WikiText-103 \citep{merity2016wikitext}, OpenWebText \citep{gokaslan2019openwebtext}, and the English split of the Colossal Clean Crawled Corpus (C4) \citep{raffel2020t5}, using the Generative Pre-trained Transformer 2 (GPT-2) tokenizer \citep{radford2019language}.

For the direct OpenWebText pretraining evaluation, we use the first 200,000 training examples for training, the next 4,096 examples for validation, and the following 4,096 examples for test.  For the direct C4 pretraining evaluation, we use the first 100,000 training examples for training and two disjoint 4,096-example slices, one for validation and one for testing.

The 22M shared-block model has a hidden size of 384, 6 heads, a Swish-gated linear unit (SwiGLU) width of 1536, a sequence length of 128 unless noted, a dropout probability of 0.1, tied embeddings, and one recurrent Transformer update block.  The 50M shared-block model uses a hidden size of 768, 12 heads, and a width of 3072.  The larger WikiText-103 scale evaluation uses a hidden size of 1280, 20 heads, and a width of 5120, yielding 95.6M trainable parameters for the shared-block LM.

\paragraph{Compute and Token-Budget Accounting.} We use an explicit per-token Transformer-body FLOP proxy with a dense causal-mask attention count:
\begin{equation}
	\mathrm{FLOPs}_{\rm body,token}(T)
	=
	N_{\rm body}(T)(8d^2+6dd_{\rm ff}+4Ld),
\end{equation}
where $d$ is the hidden width, $d_{\rm ff}$ is the SwiGLU feed-forward width, and $L$ is the sequence length.  Let $c_{\rm body}$ be the number of physical shared-block body evaluations per logical recurrent step.  For shared-block looped Transformers, $N_{\rm body}(T)=c_{\rm body}T$.  The main method and ordinary baselines have $c_{\rm body}=1$, while subtractive controls that evaluate the body both at the current state and at the anchor have $c_{\rm body}=2$.  For the unshared Transformer baseline, $N_{\rm body}(T)=\min(T,N_{\rm layers})$.  Adding the tied vocabulary-projection matrix multiplication gives
\begin{equation}
	\begin{aligned}
		\mathrm{FLOPs}_{\rm body+head,token}(T)
		 & =\mathrm{FLOPs}_{\rm body,token}(T) \\
		 & +2d|\mathcal V|,
	\end{aligned}
\end{equation}
where $|\mathcal V|$ is the GPT-2 tokenizer vocabulary size.  The $4Ld$ term is the dense $L\times L$ causal-mask attention accounting term per token, not an ideal triangular-kernel attention count.  Neither expression is a total end-to-end LM FLOP count.  Both omit normalization, softmax, nonlinearities, elementwise operations, and architecture-specific anchor, initialization, and source-projection paths.  The first also omits the vocabulary head.  We therefore use these quantities only as dominant matrix-multiplication proxies and report measured accelerator latency separately.

Under this accounting, the 22M model has 39.3M body-proxy FLOPs per token and 77.9M body-plus-head-proxy FLOPs per token at $T=8$.  At $T=48$, the 22M model has 235.9M body-proxy FLOPs per token and 274.5M body-plus-head-proxy FLOPs per token.

The 50M WikiText-103 long-training setting uses a batch size of 16 with a gradient-accumulation factor of 2, so its 5k-step schedule sees 20.48M tokens.  This 20.48M-token schedule matches the 22M 5k-step context-128 training-token budget.  The 95.6M setting uses the same effective batch and accumulation schedule, again matching the 20.48M-token 5k-step budget.  The 95.6M 10k-step longer-training comparison doubles the training-token budget to 40.96M tokens.  Under the same accounting, the 95.6M body-plus-head proxy is 553.3M FLOPs per token at $T=8$.

We also train two parameter-matched unshared controls: a 4-layer 22.7M Transformer with a width of 320 and a 6-layer 22.5M Transformer with a width of 288.  The long-context, direct web-corpus pretraining, held-out web-text transfer, and 22M LAMBADA evaluations additionally include an unshared Transformer-8 fixed-depth reference.  This model is reported as an absolute unshared-depth reference rather than a parameter-matched shared-block control.  Its active depth saturates once $T\ge 8$.

We use adaptive moment estimation with decoupled weight decay (AdamW) \citep{loshchilov2019decoupled} with $(\beta_1,\beta_2)=(0.9,0.95)$, a weight decay of 0.1, gradient clipping at 1.0, cosine learning-rate decay, and bfloat16 autocast on graphics processing unit (GPU) accelerators.  Larger 95.6M WikiText-103 runs use the same effective global batch and gradient-accumulation schedule as the matched training-budget comparisons.

We additionally evaluate the trained WikiText-103 model weights on held-out OpenWebText and C4 validation slices and on the LAMBADA last-word completion benchmark \citep{paperno2016lambada}, all with no additional training or adaptation.  For Appendix Table~\ref{tab:heldout_transfer}, the held-out OpenWebText slice immediately follows the 200,000-example direct-pretraining block, and the held-out C4 slice uses the first 4,096 validation examples.

For adaptive-depth evaluation, we evaluate the trained WikiText-103 model weights under a 48-loop ceiling.  We calibrate a max-token hidden-update stopping rule on validation data to target mean loop budgets of 8, 16, and 24, and report effective active-depth statistics on the test set.  The adaptive-depth tables use offline full-trace accounting.

For LAMBADA, target-token completion PPL and target-token accuracy provide the most informative comparison at the evaluated model scales; exact last-word accuracy is tracked in the evaluation records.

\subsubsection{Seed Protocol and Aggregation}
\label{app:seed-protocol}
The default stochastic replication set uses three fixed random seeds.  The expanded replication set adds two further fixed random seeds, giving a five-seed set.  For the 95.6M WikiText-103 setting, training offsets stochastic random-number streams while keeping the sampler seed fixed.  Appendix Table~\ref{tab:scale_profile} uses one representative trained model from the default replication set for each reported method--scale pair.  Its PPL columns come from the full evaluations of those models, while latency, throughput, and peak memory come from short measurements on the profiled accelerator setup.

The expanded set is used for the primary direct comparisons.  These include the 22M WikiText-103 direct comparisons between the tuned adapter and \SCSE{}, the secondary leak diagnostics, and the 22M recurrent-step-conditioned comparison block.  They also include the 50M and 95.6M source-centered and recurrent-step-conditioned evaluations, the 95.6M 10k-step comparison between \SCSE{} and the recurrent-step-conditioned adapter, and the context-1024 baseline and \SCSE{} evaluations.  Direct C4 tuned-adapter and \SCSE{} pretraining evaluations also use the expanded set, as do the 95.6M LAMBADA \SCSE{} and recurrent-step-conditioned evaluations.

Evaluations outside the expanded-set list use the default three-seed set.  In particular, the direct OpenWebText pretraining evaluations, held-out OpenWebText and C4 transfer evaluations, and the main WikiText-2 and broad scale evaluations use the default set.

Evaluation identities are defined by the model configuration, training protocol, dataset and split, evaluation loop depth, seed, evaluation tag, and any adaptive-runtime or calibration setting that changes the reported metric.  Aggregation keeps one semantic evaluation per identity, so repeated benchmark measurements can update stored values without increasing the effective seed count.  Integrity checks verify that each reported summary uses the required seed set and does not count duplicate evaluations within the same phase, family, seed, loop depth, dataset, split, evaluation tag, and adaptive runtime mode.

All the reported PPL means are computed after token-weighted mean negative log-likelihood aggregation and then exponentiation.  Standard deviations are reported across the available seed set for each evaluation.  For paired comparisons, summaries match evaluations by seed and loop depth before taking deltas.  Reported paired deltas use small-sample 95\% $t$-interval half-widths when shown, with deterministic paired-bootstrap percentile intervals as additional robustness checks.  These intervals are paired robustness checks rather than large-sample significance claims.  Broad sweeps typically use the default three-seed set, while the direct comparisons listed above use the expanded five-seed set.

\subsubsection{Reference PyTorch Implementation}
Listing~\ref{lst:scse-code} gives an example implementation of the source-centered recurrent update used by \SCSE{} and by the secondary leak diagnostic.  The default arguments mirror Eq.~\ref{eq:source}.  In the main method, the recurrent block receives the deviation alone and is zero-preserving because its projections are bias-free.  The per-example mask additionally makes the pointwise update exactly zero at the anchor.  The optional leak and anchor-conditioning arguments reproduce the secondary diagnostic.  For \SCSE{}, the deviation-leak coefficient is fixed at $\lambda_{\rm leak}=0$; a nonzero coefficient appears only in the secondary diagnostic.

In the experiments, \texttt{core} is the shared update block with RMSNorm pre-normalization, causal self-attention, and SwiGLU updates.  \texttt{cond\_proj} implements $W_c$, \texttt{embed}, \texttt{anchor\_proj}, and \texttt{init\_delta\_proj} are learned modules, and \texttt{readout} applies the final RMSNorm and tied output head.  Relative to the mathematical notation, \texttt{anchor\_scale * anchor\_proj} implements $a_\omega$, and \texttt{e + init\_delta\_scale * init\_delta\_proj(e)} implements $H_0(e)$.

The line computing \texttt{anchor} is executed once before the loop, so \texttt{anchor\_proj} defines the fixed reference point; the loop evolves \texttt{h} in deviation coordinates around that reference rather than repeatedly regenerating or adding the anchor.

\begin{lstlisting}[style=pythoncode,float=t!,caption={PyTorch code for \SCSE{}.  Set \texttt{cond\_proj=None}, \texttt{kappa=0}, and \texttt{leak=0} for \SCSE{}; for the secondary leak diagnostic, set \texttt{leak=0.02}, let \texttt{cond\_proj} implement $W_c$, and supply that diagnostic's learned scalar gain $\kappa$, initialized to the same configured value as the baseline's learned gain $\alpha$.},label={lst:scse-code}]
def scse_step(h, anchor, core, *, step_scale=0.5, leak=0.0,
              cond_proj=None, kappa=0.0, eps=1e-8):
    delta = h - anchor
    recurrent_input = delta
    if cond_proj is not None and kappa != 0.0:
        recurrent_input = delta + kappa * cond_proj(anchor)

    q = step_scale * core(recurrent_input)
    delta_norm_sq = delta.pow(2).sum(dim=(1, 2), keepdim=True)
    active = (delta_norm_sq > eps).to(q.dtype)
    q = q * active

    return anchor + (1.0 - leak) * delta + q


def scse_unroll(input_ids, *, embed, anchor_proj, init_delta_proj,
                core, readout, loops, anchor_scale=0.1,
                init_delta_scale=0.1, **step_kwargs):
    e = embed(input_ids)
    anchor = e + anchor_scale * anchor_proj(e)
    h = e + init_delta_scale * init_delta_proj(e)
    states = [h]

    for _ in range(loops):
        h = scse_step(h, anchor, core, **step_kwargs)
        states.append(h)

    return readout(h), states, anchor
\end{lstlisting}

\subsubsection{Verification Entry Points}
\label{app:verification-entry-points}
We verify the reported evaluations at the level of semantic evaluation identities, seed-level aggregation, unequal-batch accounting, and token-weighted metric aggregation.  Evaluation identities include the model family, dataset, split, loop depth, seed, training protocol, and adaptive-runtime setting when applicable; repeated measurements are deduplicated before computing reported summaries.

Large-data reproduction requires access to the listed corpora and suitable accelerators, while the fixed schedules, aggregation rules, and evaluation identities above define how the reported summaries are constructed.

\subsection{Additional Fixed-Depth Main Comparisons}
\label{app:fixed-depth-main-comparisons}
\paragraph{Strong Shared-Block Transformer Baselines.} Appendix Table~\ref{tab:strong_followup} compares stronger shared-block Transformer baselines on WikiText-103 at the 22M scale with 5k training steps.  The capacity-matched and tuned adapters are much stronger than the looped Transformer baseline, with the tuned adapter reaching $156.4\pm0.6$ PPL at $T=8$.

\SCSE{} improves on the tuned adapter at every reported loop depth, both within the training loop-depth range and under deep extra-loop evaluation.  These results support the full source-centered reparameterization.  Appendix Table~\ref{tab:add-causal-forcing} provides the more specific forcing intervention.

\begin{table*}[t!]
	\centering
	\scriptsize
	\begin{tabular}{lrrrr}
		\toprule
		                                  & \multicolumn{4}{c}{Test PPL}                                                                            \\
		\cmidrule(lr){2-5}
		Method                            & $T=8$                        & $T=16$                 & $T=32$                 & $T=48$                 \\
		\midrule
		Looped Transformer baseline       & $179.8\pm0.8$                & $187.3\pm0.6$          & $203.3\pm0.5$          & $218.0\pm0.8$          \\
		Capacity-matched adapter          & $157.6\pm0.5$                & $165.2\pm0.5$          & $183.9\pm1.0$          & $202.3\pm1.8$          \\
		Tuned adapter, $s=0.35^{\dagger}$ & $156.4\pm0.6$                & $164.5\pm0.5$          & $184.4\pm0.8$          & $204.1\pm1.4$          \\
		\SCSE{}, $s=0.50$                 & $\mathbf{155.1\pm0.6}$       & $\mathbf{162.0\pm0.6}$ & $\mathbf{180.7\pm0.9}$ & $\mathbf{200.1\pm1.3}$ \\
		\bottomrule
	\end{tabular}
	\caption{Strong shared-block baseline comparison on WikiText-103, 22M, 5k steps.  Rows marked $^\dagger$ are explicitly tuned 22M hyperparameter variants rather than default recipe rows.  Within the training loop-depth range, the tuned adapter is a stronger non-\SCSE{} baseline than the capacity-matched adapter, while \SCSE{} has the lowest PPL among shared-block methods across the reported depths.}
	\label{tab:strong_followup}
\end{table*}

\paragraph{Scale and Training-Budget Evaluations.} The larger-scale comparisons focus on \SCSE{} as the reported source-centered method.  The direct extra-loop evaluations below show that \SCSE{} is stronger at deep evaluation depths.

For the stronger 50M WikiText-103 setting with a gradient-accumulation factor of 2, Appendix Table~\ref{tab:scale50m_x48} extends evaluation to $T=48$ while keeping the main comparison focused on \SCSE{}.  The paired PPL deltas for \SCSE{} relative to the looped Transformer baseline are $-28.0\pm2.6$ at $T=8$, $-26.9\pm3.4$ at $T=24$, and $-22.2\pm5.5$ at $T=48$.

\begin{table*}[t!]
	\centering
	\scriptsize
	\begin{tabular}{lrrrrr}
		\toprule
		                            & \multicolumn{5}{c}{Test PPL}                                                                                                     \\
		\cmidrule(lr){2-6}
		Method                      & $T=8$                        & $T=16$                 & $T=24$                 & $T=32$                 & $T=48$                 \\
		\midrule
		Looped Transformer baseline & $151.1\pm0.9$                & $156.4\pm0.9$          & $162.5\pm1.0$          & $168.3\pm1.0$          & $178.9\pm1.3$          \\
		\SCSE{}                     & $\mathbf{123.1\pm0.2}$       & $\mathbf{128.6\pm0.2}$ & $\mathbf{135.5\pm0.4}$ & $\mathbf{142.6\pm0.6}$ & $\mathbf{156.4\pm1.1}$ \\
		\bottomrule
	\end{tabular}
	\caption{Extended 50M WikiText-103 5k-step comparison with evaluation through $T=48$ on the same 20.48M-token training budget.  The table reports the source-centered comparison: \SCSE{} improves on the looped Transformer baseline at every reported loop depth.}
	\label{tab:scale50m_x48}
\end{table*}

At 95.6M, we test whether the source-centered advantage over stronger non-\SCSE{} controls persists.  Appendix Table~\ref{tab:scale95m} shows that a tuned adapter cuts $T=8$ PPL from $117.1\pm0.5$ to $98.6\pm0.5$, while \SCSE{} remains best at $96.9\pm0.3$.  The same ordering holds at $T=16$, $T=24$, and $T=32$; at $T=32$, \SCSE{} reaches $113.7\pm0.7$ compared with $116.1\pm0.9$ for the tuned adapter.  Within the training loop-depth range, mean PPL improves from $T=4$ to $T=8$ for all three reported methods.  Mean PPL decreases from $117.4$ to $117.1$ for the looped Transformer baseline, from $98.9$ to $98.6$ for the tuned adapter, and from $97.3$ to $96.9$ for \SCSE{}.  The improvement within the training loop-depth range grows with the training budget.  Under the 10k-step protocol of Appendix Table~\ref{tab:long10k_95m}, the mean improvements between $T=4$ and $T=8$ are $0.7$ PPL for the recurrent-step-conditioned adapter and $0.8$ PPL for \SCSE{}.  Loop-depth scaling therefore helps up to the boundary of the training loop-depth range, and the degradation at $T>8$ is specific to evaluating the same trained weights at extrapolated depths.

\begin{table*}[t!]
	\centering
	\scriptsize
	\begin{tabular}{lrrrrr}
		\toprule
		                            & \multicolumn{5}{c}{Test PPL}                                                                                                    \\
		\cmidrule(lr){2-6}
		Method                      & $T=4$                        & $T=8$                 & $T=16$                 & $T=24$                 & $T=32$                 \\
		\midrule
		Looped Transformer baseline & $117.4\pm0.5$                & $117.1\pm0.5$         & $122.1\pm0.6$          & $128.0\pm0.7$          & $133.9\pm0.9$          \\
		Tuned adapter, $s=0.35$     & $98.9\pm0.3$                 & $98.6\pm0.5$          & $103.7\pm0.7$          & $109.8\pm0.9$          & $116.1\pm0.9$          \\
		\SCSE{}                     & $\mathbf{97.3\pm0.2}$        & $\mathbf{96.9\pm0.3}$ & $\mathbf{101.7\pm0.5}$ & $\mathbf{107.6\pm0.6}$ & $\mathbf{113.7\pm0.7}$ \\
		\bottomrule
	\end{tabular}
	\caption{Larger-scale WikiText-103 strong-baseline comparison at the 95.6M shared-block scale, with models trained for 5k steps on 20.48M tokens using the same effective batch and gradient-accumulation schedule as the matched training-budget comparisons.  Each entry reports the mean test PPL $\pm$ one standard deviation.  Boldface marks the best value in the relevant comparison; when an unshared fixed-depth reference row is included, boldface is restricted to the shared-block rows unless stated otherwise.  Unless a caption states otherwise, result tables use the stochastic replication sets defined by the replication protocol.  Among the reported non-\SCSE{} controls, the tuned adapter is strongest within the training loop-depth range ($T=4,8$), but \SCSE{} gives the lowest PPL at every reported depth.  For every method, $T=8$ PPL is lower than $T=4$ PPL, so loop-depth scaling helps within the training loop-depth range, and degradation appears only at the extrapolated depths $T>8$.  Under the stated accounting, the body-plus-head proxy is 553.3M FLOPs per token at $T=8$.}
	\label{tab:scale95m}
\end{table*}

\subsection{Longer Training at 95.6M}
\label{app:longer-training-95m}

Appendix Table~\ref{tab:long10k_95m} evaluates the 95.6M \SCSE{} advantage under a longer training protocol rather than only under the 5k-step training-token budget.  We double the training-token budget from 20.48M to 40.96M tokens, while keeping the same loop-depth sampler, optimizer, effective batch schedule, and recurrent-step-conditioned adapter hyperparameters.

Both methods improve substantially in absolute PPL.  The recurrent-step-conditioned adapter reaches $68.5\pm0.2$ PPL at $T=8$ and $103.4\pm0.8$ at $T=48$, which is much stronger than its 5k-step counterpart.  \SCSE{} nevertheless remains better at every measured depth, reaching $66.7\pm0.3$ at $T=8$, $76.9\pm0.5$ at $T=24$, and $98.3\pm1.0$ at $T=48$.

The paired \SCSE{} minus recurrent-step-conditioned deltas are $-1.80$, $-3.45$, and $-5.10$ PPL at $T=8,24,48$, with deterministic paired-bootstrap intervals $[-2.13,-1.49]$, $[-3.85,-2.98]$, and $[-6.11,-3.99]$.  The longer-training comparison reinforces that recurrent-step conditioning alone does not account for the source-centered quality advantage.

\begin{table*}[t!]
	\centering
	\scriptsize
	\begin{tabular}{lrrrrr}
		\toprule
		                                   & \multicolumn{5}{c}{Test PPL}                                                                                                 \\
		\cmidrule(lr){2-6}
		Method                             & $T=4$                        & $T=8$                 & $T=16$                & $T=24$                & $T=48$                \\
		\midrule
		Recurrent-step-conditioned adapter & $69.2\pm0.2$                 & $68.5\pm0.2$          & $73.6\pm0.3$          & $80.4\pm0.3$          & $103.4\pm0.8$         \\
		\SCSE{}                            & $\mathbf{67.5\pm0.3}$        & $\mathbf{66.7\pm0.3}$ & $\mathbf{70.9\pm0.4}$ & $\mathbf{76.9\pm0.5}$ & $\mathbf{98.3\pm1.0}$ \\
		\bottomrule
	\end{tabular}
	\caption{Longer-training comparison at 95.6M on WikiText-103.  Both methods follow the 95.6M training protocol in Table~\ref{tab:stepcond_wt103}, but are trained for 10k steps and 40.96M tokens instead of 5k steps and 20.48M tokens.  Each entry reports the mean test PPL $\pm$ one standard deviation.  \SCSE{} preserves its advantage over the recurrent-step-conditioned adapter with the larger training-token budget.}
	\label{tab:long10k_95m}
\end{table*}

\subsection{Adaptive-Depth and Runtime}
\label{app:adaptive-runtime}
\paragraph{Adaptive-Depth Evaluation.} A practical advantage of shared recurrent depth is that a single set of trained weights can expose a continuous inference loop-budget knob rather than a single fixed loop depth.  The adaptive-depth evaluation tests that knob with inference-time, example-level stopping.  Appendix Table~\ref{tab:adaptive_wt103} evaluates this property using a stronger 22M comparison set consisting of the looped Transformer baseline, the tuned adapter, the recurrent-step-conditioned adapter, and \SCSE{}.  We calibrate a max-token hidden-update stopping rule on validation data to target mean loop budgets of 8, 16, and 24 under a 48-loop ceiling and then evaluate on the test set.

The result preserves the fixed-depth ordering under loop-budgeted inference.  \SCSE{} is the best shared-block method at target mean loop budgets of 8, 16, and 24.  \SCSE{} reaches $156.0$, $163.3$, and $171.8$ PPL.  The tuned adapter reaches $157.9$, $165.4$, and $174.1$, and the recurrent-step-conditioned adapter reaches $158.3$, $165.8$, and $175.1$.

\begin{table*}[t!]
	\centering
	\scriptsize
	\begin{tabular}{lrrrrr}
		\toprule
		                                   & \multicolumn{5}{c}{PPL}                                                                                   \\
		\cmidrule(lr){2-6}
		Method                             & Static $T=8$            & Adaptive 8       & Adaptive 16      & Adaptive 24      & Static $T=48$          \\
		\midrule
		Looped Transformer baseline        & $179.8\pm0.8$           & 182.6            & 188.7            & 195.3            & $218.0\pm0.8$          \\
		Tuned adapter, $s=0.35$            & $156.4\pm0.6$           & 157.9            & 165.4            & 174.1            & $204.1\pm1.4$          \\
		Recurrent-step-conditioned adapter & $156.7\pm0.5$           & 158.3            & 165.8            & 175.1            & $206.1\pm1.4$          \\
		\SCSE{}, $s=0.50$                  & $\mathbf{155.1\pm0.6}$  & $\mathbf{156.0}$ & $\mathbf{163.3}$ & $\mathbf{171.8}$ & $\mathbf{200.1\pm1.3}$ \\
		\bottomrule
	\end{tabular}
	\caption{Validation-calibrated adaptive-depth evaluation on WikiText-103, 22M, 5k steps, with a 48-loop ceiling and a simple max-token hidden-update stopping rule.  We report the stronger comparison set including tuned and recurrent-step-conditioned adapters.  Thresholds are chosen on validation to target mean loop budgets of 8, 16, and 24, and performance is then evaluated on the test set.  \SCSE{} remains the strongest shared-block choice when the reusable loop-budget knob is evaluated under per-example stopping.}
	\label{tab:adaptive_wt103}
\end{table*}

The same operating-regime split holds when we scale the loop-budgeted evaluation.  Appendix Table~\ref{tab:adaptive_scale} reports the corresponding adaptive-depth evaluation on the 50M and 95.6M comparison sets with target mean loop budgets of 8, 16, and 24 under the same 48-loop ceiling.  The thresholds remain reasonably calibrated: most achieved mean active depths are within about $0.4$ loops of the target, and the largest reported deviation is about $1.2$ loops.

At 50M, \SCSE{} remains the best shared-block method at every target mean loop budget.  \SCSE{} reaches $123.5$, $128.9$, and $135.5$ PPL at target mean loop budgets of 8, 16, and 24.  The recurrent-step-conditioned adapter reaches $126.4$, $131.9$, and $138.6$.  At 95.6M, \SCSE{} again remains best.  \SCSE{} reaches $97.1$, $101.3$, and $106.6$.  The tuned adapter reaches $98.9$, $103.1$, and $108.4$, and the recurrent-step-conditioned adapter reaches $101.5$, $105.9$, and $111.0$.

These scale-up adaptive-depth evaluations reinforce the same message as the fixed-depth results.  Source-centered coordinates, not recurrent-step conditioning alone, yield the best quality under a reusable loop-budget knob.

\begin{table*}[t!]
	\centering
	\scriptsize
	\begin{tabular}{lrrrrr}
		\toprule
		                                   & \multicolumn{5}{c}{PPL}                                                                                   \\
		\cmidrule(lr){2-6}
		Method                             & Static $T=8$            & Adaptive 8       & Adaptive 16      & Adaptive 24      & Static $T=48$          \\
		\midrule
		\multicolumn{6}{l}{50M}                                                                                                                        \\
		Looped Transformer baseline        & $151.1\pm0.9$           & 153.8            & 158.3            & 162.8            & $178.9\pm1.3$          \\
		Recurrent-step-conditioned adapter & $125.7\pm0.4$           & 126.4            & 131.9            & 138.6            & $160.1\pm0.2$          \\
		\SCSE{}                            & $\mathbf{123.1\pm0.2}$  & $\mathbf{123.5}$ & $\mathbf{128.9}$ & $\mathbf{135.5}$ & $\mathbf{156.4\pm1.1}$ \\
		\midrule
		\multicolumn{6}{l}{95.6M}                                                                                                                      \\
		Looped Transformer baseline        & $117.1\pm0.5$           & 117.8            & 122.2            & 127.6            & $145.1\pm1.2$          \\
		Tuned adapter, $s=0.35$            & $98.6\pm0.5$            & 98.9             & 103.1            & 108.4            & $128.7\pm0.9$          \\
		Recurrent-step-conditioned adapter & $100.8\pm0.5$           & 101.5            & 105.9            & 111.0            & $130.6\pm1.0$          \\
		\SCSE{}                            & $\mathbf{96.9\pm0.3}$   & $\mathbf{97.1}$  & $\mathbf{101.3}$ & $\mathbf{106.6}$ & $\mathbf{125.9\pm0.9}$ \\
		\bottomrule
	\end{tabular}
	\caption{Scale-up adaptive-depth evaluation on WikiText-103 comparison sets.  All methods use a 48-loop ceiling and a validation-calibrated max-token hidden-update stopping rule with target mean loop budgets of 8, 16, and 24.  Values are test means.  Most achieved mean active depths are within about 0.4 loops of the targets, and the largest deviation is about 1.2 loops.}
	\label{tab:adaptive_scale}
\end{table*}

\paragraph{Accelerator Profiling.} Appendix Table~\ref{tab:scale_profile} reports GPU accelerator measurements for one trained model per reported method--scale pair at the larger 50M and 95.6M WikiText-103 scales.  At 50M, \SCSE{} is the fastest quality-competitive shared-block method both at the boundary of the training loop-depth range, $T=8$, and at the deep extra-loop depth $T=48$.  At $T=8$, \SCSE{} reaches $123.3$ PPL at 17.6 milliseconds (ms) per batch, compared with $125.3$ at 19.2 ms for the recurrent-step-conditioned adapter and $150.1$ at 24.2 ms for the looped Transformer baseline.  At $T=48$, \SCSE{} reaches $158.2$ PPL at 77.8 ms compared with $160.4$ at 87.5 ms for the recurrent-step-conditioned adapter.

At 95.6M, \SCSE{} is the best-quality shared-block method and remains close to the fastest measured method.  \SCSE{} reaches $96.4$ PPL at $T=8$ in 19.6 ms per batch and $124.3$ at $T=48$ in 89.9 ms per batch.  The tuned adapter reaches $98.0$ and $127.7$ PPL in $20.0$ and $92.9$ ms, respectively, and the recurrent-step-conditioned adapter reaches $100.7$ and $130.0$ PPL in $21.2$ and $100.1$ ms, respectively.  Peak allocated accelerator memory, reported in megabytes (MB), remains closely clustered at each scale, so the main practical distinction among strong shared-block methods is quality and loop-step overhead rather than a large memory gap.

\begin{table*}[t!]
	\centering
	\scriptsize
	\resizebox{\textwidth}{!}{%
		\begin{tabular}{lrrrrrr}
			\toprule
			                                   & \multicolumn{2}{c}{$T=8$} & \multicolumn{4}{c}{$T=48$ profile}                                                                                             \\
			\cmidrule(lr){2-3}\cmidrule(lr){4-7}
			Method                             & PPL                       & Latency (ms)                       & PPL            & Latency (ms)  & Throughput ($10^3$ tokens per second) & Peak memory (MB) \\
			\midrule
			\multicolumn{7}{l}{50M}                                                                                                                                                                         \\
			Looped Transformer baseline        & $150.1$                   & $24.2$                             & $178.4$        & \textbf{75.2} & \textbf{27.2}                         & \textbf{1365.5}  \\
			Recurrent-step-conditioned adapter & $125.3$                   & $19.2$                             & $160.4$        & $87.5$        & $23.4$                                & $1379.8$         \\
			\SCSE{}                            & $\mathbf{123.3}$          & \textbf{17.6}                      & \textbf{158.2} & $77.8$        & $26.3$                                & $1373.3$         \\
			\midrule
			\multicolumn{7}{l}{95.6M}                                                                                                                                                                       \\
			Looped Transformer baseline        & $117.3$                   & \textbf{18.3}                      & $145.4$        & \textbf{86.9} & \textbf{11.8}                         & \textbf{1268.8}  \\
			Tuned adapter, $s=0.35$            & $98.0$                    & $20.0$                             & $127.7$        & $92.9$        & $11.0$                                & $1284.3$         \\
			Recurrent-step-conditioned adapter & $100.7$                   & $21.2$                             & $130.0$        & $100.1$       & $10.2$                                & $1294.6$         \\
			\SCSE{}                            & $\mathbf{96.4}$           & $19.6$                             & \textbf{124.3} & $89.9$        & $11.4$                                & $1281.4$         \\
			\bottomrule
		\end{tabular}
	}%
	\caption{Accelerator profile for one trained model per reported method--scale pair at the larger WikiText-103 scales.  The looped Transformer baseline is fastest at 50M $T=48$ but has much worse PPL.  \SCSE{} is the fastest quality-competitive shared-block method at 50M and remains competitive with the fastest 95.6M method while also being the best-quality shared-block method.}
	\label{tab:scale_profile}
\end{table*}

\subsection{Context and Corpus Extension}
\label{app:context-corpus-extension}
\paragraph{Longer Contexts.} Appendix Table~\ref{tab:contexts} revisits the longer-context setting under a 5k-step training protocol at sequence lengths 512 and 1024, using batch sizes 8 and 4, respectively, so that each optimizer step still sees 4096 tokens.  Source-centered coordinates remain beneficial in the longer-context setting.

At a context length of 512, the \SCSE{} minus tuned-adapter paired PPL deltas are $-3.34\pm3.81$ at $T=8$, $-6.45\pm2.91$ at $T=24$, and $-9.41\pm3.06$ at $T=48$.  At a context length of 1024, the same paired deltas are $-4.15\pm3.17$, $-5.94\pm3.38$, and $-5.00\pm6.31$.  Thus, the longer-context table focuses the primary shared-block analysis on comparisons of \SCSE{} with the tuned adapter and looped Transformer baseline.

The unshared Transformer-8 fixed-depth reference is included as an absolute unshared-depth reference rather than a parameter-matched shared-block control.  The reference is still the absolute PPL reference at $T=8$.  The reference reaches $153.7\pm0.8$ and $181.6\pm1.9$ at context lengths of 512 and 1024, respectively.  However, the reference requires roughly 77 MB rather than 45 MB of 16-bit parameter memory and trains more slowly than the shared-block source-centered method.

\begin{table*}[t!]
	\centering
	\small
	\begin{tabular}{lrrrrr}
		\toprule
		                                              & \multicolumn{5}{c}{Test PPL}                                                                                                     \\
		\cmidrule(lr){2-6}
		Method                                        & $T=4$                        & $T=8$                  & $T=16$                 & $T=24$                 & $T=48$                 \\
		\midrule
		\multicolumn{6}{l}{Context 512}                                                                                                                                                  \\
		Looped Transformer baseline                   & $229.8\pm2.1$                & $230.1\pm1.9$          & $237.8\pm1.6$          & $246.6\pm1.4$          & $273.8\pm0.7$          \\
		Tuned adapter, $s=0.35$                       & $193.3\pm1.2$                & $191.8\pm1.2$          & $200.7\pm1.2$          & $212.3\pm1.2$          & $249.5\pm1.6$          \\
		\SCSE{}                                       & $\mathbf{190.8\pm0.5}$       & $\mathbf{188.4\pm0.5}$ & $\mathbf{195.6\pm0.5}$ & $\mathbf{205.9\pm0.6}$ & $\mathbf{240.1\pm1.7}$ \\
		\addlinespace[2pt]
		Unshared Transformer-8, fixed-depth reference & not reported                 & $153.7\pm0.8$          & $153.7\pm0.8$          & $153.7\pm0.8$          & $153.7\pm0.8$          \\
		\midrule
		\multicolumn{6}{l}{Context 1024}                                                                                                                                                 \\
		Looped Transformer baseline                   & $264.2\pm2.0$                & $264.9\pm2.1$          & $272.5\pm2.1$          & $281.1\pm2.1$          & $308.5\pm2.1$          \\
		Tuned adapter, $s=0.35$                       & $226.5\pm0.5$                & $224.7\pm0.6$          & $232.9\pm0.7$          & $243.8\pm0.7$          & $278.6\pm0.7$          \\
		\SCSE{}                                       & $\mathbf{223.7\pm1.3}$       & $\mathbf{221.1\pm1.2}$ & $\mathbf{228.0\pm1.1}$ & $\mathbf{238.1\pm1.0}$ & $\mathbf{272.6\pm2.1}$ \\
		\addlinespace[2pt]
		Unshared Transformer-8, fixed-depth reference & not reported                 & $181.6\pm1.9$          & $181.6\pm1.9$          & $181.6\pm1.9$          & $181.6\pm1.9$          \\
		\bottomrule
	\end{tabular}
	\caption{WikiText-103 longer-context evaluation, 22M, 5k steps.  \SCSE{} is the best shared-block method at both context lengths and every reported loop depth.  The Unshared Transformer-8 row is an absolute fixed-depth reference rather than a parameter-matched shared-block control; the reference is trained with a fixed depth of 8 layers, so results for the reference are reported only from $T=8$ onward, and its active depth saturates once $T\ge 8$.}
	\label{tab:contexts}
\end{table*}

\subsubsection{Cross-Corpus and Out-of-Domain Evaluation}
\label{app:cross-corpus-ood-evaluation}
\paragraph{Direct Web-Corpus Pretraining.} To test whether the effect is WikiText-specific, Appendix Tables~\ref{tab:owt_train} and~\ref{tab:c4_train} report the full strong-baseline suite at the 22M scale with 5k training steps under direct web-corpus pretraining on two web-text corpora: OpenWebText and English C4.

On OpenWebText, \SCSE{} is strongest among reported shared-block methods at every reported loop depth.  The paired \SCSE{} minus tuned-adapter deltas are $-1.96\pm1.98$, $-5.12\pm3.13$, and $-10.42\pm14.41$ PPL at $T=8,24,48$, so the OpenWebText result is consistent with the main source-centered comparison but is noisier at the deepest setting.

On English C4 \citep{raffel2020t5}, the same primary comparison is clearer.  \SCSE{} is strongest among reported shared-block methods at $T=8$, $T=24$, and $T=48$.  The paired \SCSE{} minus tuned-adapter deltas are $-3.32\pm1.28$, $-7.12\pm1.30$, and $-8.74\pm2.36$ PPL at $T=8,24,48$, with paired-bootstrap intervals $[-3.97,-2.40]$, $[-7.94,-6.23]$, and $[-10.47,-7.60]$.  C4 therefore provides a cleaner cross-corpus check of the source-centered pattern.

\paragraph{Cross-Corpus and Out-of-Domain Tables}
\label{app:cross-corpus-ood}
\paragraph{Direct Web-Corpus Pretraining Tables.} Appendix Tables~\ref{tab:owt_train} and~\ref{tab:c4_train} provide the tabulated direct web-corpus pretraining results referenced in the main text.

\begin{table*}[t!]
	\centering
	\small
	\begin{tabular}{lrrrr}
		\toprule
		                                              & \multicolumn{4}{c}{PPL}                                                                            \\
		\cmidrule(lr){2-5}
		Method                                        & $T=8$                   & $T=16$                 & $T=24$                 & $T=48$                 \\
		\midrule
		Looped Transformer baseline                   & $285.7\pm0.3$           & $295.3\pm0.2$          & $306.3\pm0.2$          & $336.8\pm0.5$          \\
		Tuned adapter, $s=0.35$                       & $255.3\pm0.9$           & $265.9\pm1.3$          & $279.7\pm2.2$          & $328.0\pm9.4$          \\
		\SCSE{}                                       & $\mathbf{253.3\pm0.8}$  & $\mathbf{262.4\pm1.0}$ & $\mathbf{274.5\pm1.2}$ & $\mathbf{317.5\pm3.6}$ \\
		\addlinespace[2pt]
		Unshared Transformer-8, fixed-depth reference & $222.0\pm0.8$           & $222.0\pm0.8$          & $222.0\pm0.8$          & $222.0\pm0.8$          \\
		\bottomrule
	\end{tabular}
	\caption{Direct OpenWebText pretraining, 22M, 5k steps, with the same 20.48M-token training budget as the WikiText-103 22M longer-training settings.  \SCSE{} is the strongest reported shared-block method across the reported loop depths, though the deepest paired comparison is noisier than the C4 counterpart.  The Unshared Transformer-8 row is an absolute fixed-depth reference rather than a parameter-matched shared-block control, and its active depth saturates once $T\ge 8$.}
	\label{tab:owt_train}
\end{table*}

\begin{table*}[t!]
	\centering
	\scriptsize
	\begin{tabular}{lrrrr}
		\toprule
		                                              & \multicolumn{4}{c}{PPL}                                                                            \\
		\cmidrule(lr){2-5}
		Method                                        & $T=8$                   & $T=16$                 & $T=24$                 & $T=48$                 \\
		\midrule
		Looped Transformer baseline                   & $348.8\pm3.7$           & $360.2\pm3.5$          & $372.8\pm3.4$          & $407.5\pm3.3$          \\
		Tuned adapter, $s=0.35$                       & $304.8\pm0.6$           & $318.0\pm0.6$          & $334.5\pm0.8$          & $383.2\pm2.0$          \\
		\SCSE{}                                       & $\mathbf{301.5\pm1.0}$  & $\mathbf{312.5\pm1.1}$ & $\mathbf{327.4\pm1.4}$ & $\mathbf{374.4\pm2.4}$ \\
		\addlinespace[2pt]
		Unshared Transformer-8, fixed-depth reference & $265.0\pm0.3$           & $265.0\pm0.3$          & $265.0\pm0.3$          & $265.0\pm0.3$          \\
		\bottomrule
	\end{tabular}
	\caption{Direct English C4 pretraining, 22M, 5k steps, using 100,000 training examples and disjoint 4,096-example validation and test slices.  \SCSE{} is the strongest reported shared-block method across the reported loop depths.  The Unshared Transformer-8 row is an absolute fixed-depth reference rather than a parameter-matched shared-block control, and its active depth saturates once $T\ge 8$.}
	\label{tab:c4_train}
\end{table*}

\paragraph{Held-Out Web-Text Transfer.} To separate direct web-corpus pretraining from held-out web-text transfer, Appendix Table~\ref{tab:heldout_transfer} evaluates the trained WikiText-103 model weights on held-out OpenWebText and held-out C4 slices with no additional training.  These benchmarks probe evaluation-only corpus shift on natural web text with the same GPT-2 tokenizer.

On held-out OpenWebText, the ordering within the training loop-depth range is mixed across scales, but the deep extra-loop ordering is consistent.  At 22M, \SCSE{} is best at $T=48$ with $3.87\pm0.32$k PPL.  At 50M, the recurrent-step-conditioned adapter is the best method within the training loop-depth range at $T=8$ with $3.66\pm0.33$k PPL, while \SCSE{} is best at $T=48$ with $3.47\pm0.22$k PPL.  At 95.6M, the tuned adapter is the best method within the training loop-depth range at $T=8$, while \SCSE{} is the best deep extra-loop method at $T=48$.

Held-out C4 has lower absolute PPL values than held-out OpenWebText in the held-out transfer setup and gives a more consistent primary ordering.  At 22M, \SCSE{} is best among shared-block methods at $T=48$.  At 50M, \SCSE{} is best at both depths among the reported shared-block methods.  At 95.6M, \SCSE{} remains best at both $T=8$ and $T=48$ among the reported shared-block methods.  Relative to the recurrent-step-conditioned adapter, its paired mean deltas are $-0.136$k and $-0.284$k PPL at $T=8$ and $T=48$, with paired-bootstrap intervals $[-0.207,-0.054]$k and $[-0.394,-0.137]$k.

These results support the same source-centered pattern under direct web-corpus pretraining and held-out web-text transfer, while tuned-adapter and recurrent-step-conditioned controls remain strongest in some settings within the training loop-depth range.

\paragraph{Held-Out Web-Text Transfer Table.} Appendix Table~\ref{tab:heldout_transfer} provides the tabulated held-out web-text transfer results referenced in the main text.

\begin{table*}[t!]
	\centering
	\scriptsize
	\begin{tabular}{lrrrr}
		\toprule
		                                              & \multicolumn{2}{c}{OpenWebText PPL ($10^3$)} & \multicolumn{2}{c}{C4 PPL ($10^3$)}                                                   \\
		\cmidrule(lr){2-3}\cmidrule(lr){4-5}
		Method                                        & $T=8$                                        & $T=48$                              & $T=8$                  & $T=48$                 \\
		\midrule
		\multicolumn{5}{l}{22M WikiText-103 5k trained weights}                                                                                                                              \\
		Tuned adapter, $s=0.35$                       & $\mathbf{4.03\pm0.28}$                       & $4.09\pm0.45$                       & $3.36\pm0.17$          & $3.55\pm0.25$          \\
		Recurrent-step-conditioned adapter            & $4.10\pm0.54$                                & $4.55\pm0.84$                       & $3.73\pm0.38$          & $4.19\pm0.54$          \\
		\SCSE{}                                       & $4.06\pm0.26$                                & $\mathbf{3.87\pm0.32}$              & $\mathbf{3.26\pm0.06}$ & $\mathbf{3.30\pm0.17}$ \\
		\addlinespace[2pt]
		Unshared Transformer-8, fixed-depth reference & $4.11\pm0.30$                                & $4.11\pm0.30$                       & $3.00\pm0.09$          & $3.00\pm0.09$          \\
		\midrule
		\multicolumn{5}{l}{50M WikiText-103 5k trained weights}                                                                                                                              \\
		Looped Transformer baseline                   & $3.81\pm0.11$                                & $3.94\pm0.14$                       & $3.08\pm0.11$          & $3.24\pm0.07$          \\
		Recurrent-step-conditioned adapter            & $\mathbf{3.66\pm0.33}$                       & $4.02\pm0.36$                       & $2.80\pm0.14$          & $3.14\pm0.20$          \\
		\SCSE{}                                       & $3.69\pm0.07$                                & $\mathbf{3.47\pm0.22}$              & $\mathbf{2.72\pm0.04}$ & $\mathbf{2.70\pm0.15}$ \\
		\midrule
		\multicolumn{5}{l}{95.6M WikiText-103 5k trained weights}                                                                                                                            \\
		Looped Transformer baseline                   & $3.69\pm0.13$                                & $3.75\pm0.18$                       & $2.85\pm0.01$          & $2.97\pm0.04$          \\
		Tuned adapter, $s=0.35$                       & $\mathbf{3.53\pm0.21}$                       & $3.60\pm0.23$                       & $2.68\pm0.08$          & $2.83\pm0.13$          \\
		Recurrent-step-conditioned adapter            & $3.56\pm0.14$                                & $3.70\pm0.20$                       & $2.75\pm0.07$          & $2.93\pm0.13$          \\
		\SCSE{}                                       & $3.60\pm0.20$                                & $\mathbf{3.42\pm0.31}$              & $\mathbf{2.62\pm0.04}$ & $\mathbf{2.65\pm0.10}$ \\
		\bottomrule
	\end{tabular}
	\caption{Held-out web-text benchmarks on 4,096-example OpenWebText and English C4 evaluation slices.  The OpenWebText slice immediately follows the 200,000-example direct-pretraining block, and C4 uses the first 4,096 examples of the English validation split.  All methods use their respective trained WikiText-103 weights with no extra training or adaptation.  The Unshared Transformer-8 row is an absolute fixed-depth reference rather than a parameter-matched shared-block control.  Each entry reports the mean $\pm$ one standard deviation.}
	\label{tab:heldout_transfer}
\end{table*}

\paragraph{Out-of-Domain Completion.} Appendix Table~\ref{tab:lambada} adds the LAMBADA last-word completion benchmark.  Exact last-word accuracy is near the floor at the evaluated model scales, so we emphasize target-token completion PPL and target-token accuracy on the held-out answer span.

The 22M suite preserves the main deep extra-loop advantage: \SCSE{} is the best shared-block method at $T=48$.  At 50M and 95.6M, recurrent-step conditioning improves the looped Transformer baseline and can approach the tuned adapter on target-token PPL, but \SCSE{} remains the strongest shared-block method on the main completion metrics.

At 95.6M, the paired \SCSE{} minus recurrent-step-conditioned deltas are $-1.09$k target-token completion PPL at $T=8$ and $-2.87$k at $T=32$, with bootstrap intervals $[-2.04,-0.13]$k and $[-4.43,-1.47]$k.

\paragraph{Out-of-Domain Completion Table.} Appendix Table~\ref{tab:lambada} provides the tabulated LAMBADA completion results referenced in the main text.

\begin{table*}[t!]
	\centering
	\scriptsize
	\begin{tabular}{lrrcrr}
		\toprule
		                                              & \multicolumn{2}{c}{$T=8$} & \multicolumn{3}{c}{Reported comparison depth}                                                              \\
		\cmidrule(lr){2-3}\cmidrule(lr){4-6}
		Method                                        & Completion PPL ($10^3$)   & Accuracy (\%)                                 & Depth   & Completion PPL ($10^3$) & Accuracy (\%)          \\
		\midrule
		\multicolumn{6}{l}{22M WikiText-103 5k}                                                                                                                                                \\
		Tuned adapter, $s=0.35$                       & $23.8\pm1.2$              & $1.27\pm0.19$                                 & $T=48$  & $25.9\pm2.5$            & $0.64\pm0.16$          \\
		Recurrent-step-conditioned adapter            & $30.3\pm4.3$              & $1.33\pm0.36$                                 & $T=48$  & $34.5\pm7.1$            & $0.55\pm0.15$          \\
		\SCSE{}                                       & $\mathbf{23.7\pm1.2}$     & $\mathbf{1.52\pm0.11}$                        & $T=48$  & $\mathbf{24.2\pm0.5}$   & $\mathbf{0.84\pm0.19}$ \\
		\addlinespace[2pt]
		Unshared Transformer-8, fixed-depth reference & $21.1\pm0.9$              & $1.33\pm0.06$                                 & $T\ge8$ & $21.1\pm0.9$            & $1.33\pm0.06$          \\
		\midrule
		\multicolumn{6}{l}{50M WikiText-103 5k}                                                                                                                                                \\
		Looped Transformer baseline                   & $25.0\pm1.2$              & $0.97\pm0.13$                                 & $T=48$  & $26.5\pm3.5$            & $0.59\pm0.10$          \\
		Recurrent-step-conditioned adapter            & $20.6\pm2.0$              & $2.37\pm0.09$                                 & $T=48$  & $24.3\pm2.5$            & $1.43\pm0.19$          \\
		\SCSE{}                                       & $\mathbf{18.2\pm0.1}$     & $\mathbf{2.72\pm0.32}$                        & $T=48$  & $\mathbf{19.8\pm0.8}$   & $\mathbf{1.69\pm0.15}$ \\
		\midrule
		\multicolumn{6}{l}{95.6M WikiText-103 5k}                                                                                                                                              \\
		Looped Transformer baseline                   & $23.0\pm0.3$              & $1.60\pm0.12$                                 & $T=32$  & $23.4\pm0.7$            & $1.20\pm0.25$          \\
		Tuned adapter, $s=0.35$                       & $18.5\pm1.4$              & $2.86\pm0.04$                                 & $T=32$  & $19.9\pm1.3$            & $2.13\pm0.08$          \\
		Recurrent-step-conditioned adapter            & $18.0\pm0.8$              & $2.69\pm0.17$                                 & $T=32$  & $19.9\pm1.6$            & $1.92\pm0.08$          \\
		\SCSE{}                                       & $\mathbf{16.9\pm0.8}$     & $\mathbf{3.42\pm0.16}$                        & $T=32$  & $\mathbf{17.1\pm1.1}$   & $\mathbf{2.55\pm0.16}$ \\
		\bottomrule
	\end{tabular}
	\caption{Out-of-domain LAMBADA last-word completion benchmark.  We report target-token completion PPL and target-token accuracy on the held-out answer span.  The Depth column shows the comparison depth used in each row.  The Unshared Transformer-8 row is an absolute fixed-depth reference rather than a parameter-matched shared-block control.  Exact last-word accuracy stays near the floor at these scales, with group means at most about $0.13\%$, so we tabulate the more resolved target-token completion metrics; exact last-word accuracy is tracked in the evaluation records.}
	\label{tab:lambada}
\end{table*}

\subsection{Additional Scale and Corpus Evaluations}
\label{app:additional-scale-corpus}

The C4 scale runs reuse the selected shared-block settings without additional large-corpus retuning, matching the tokenizer, optimizer, context length, loop-depth sampler, training-token budget, and evaluation depths.

The C4 scale evaluations use 136.5M shared-block models and 139.2M recurrent-step-conditioned control models with $d=1664$, 26 attention heads, $d_{\rm ff}=6656$, a context length of 128, a batch size of 32, and an AdamW learning rate of $2.5\times 10^{-4}$.  These evaluations follow the C4 split convention defined in the experimental setup details above.  The 5k- and 50k-step C4 runs correspond to training-token budgets of 20.48M and 204.8M tokens per run, respectively, and C4 test-slice PPL is reported at selected depths $T\in\{4,8,16,24,32,48,96\}$.

\subsubsection{Longer-Training and Web-Corpus Evaluations}
\label{app:scale-real-corpus}

To test whether the source-centered advantage persists beyond the 5k-step WikiText-103 setting, Appendix Table~\ref{tab:long10k_95m} reports longer WikiText-103 training at 95.6M parameters, while direct 22M OpenWebText and C4 pretraining are reported in Appendix Tables~\ref{tab:owt_train} and~\ref{tab:c4_train}.  The additional scale evaluation below trains 136.5M \SCSE{} models on C4 and compares the 136.5M \SCSE{} models with tuned and recurrent-step-conditioned adapter controls.  The largest recurrent-step-conditioned control has 139.2M parameters, and no additional large-corpus retuning is applied.

Appendix Table~\ref{tab:add-scale-real-corpus} extends the web-corpus comparison to the 136.5M C4 5k-step setting.  \SCSE{} improves on the tuned adapter by 5.83 PPL at $T=8$, 6.83 PPL at $T=24$, and 9.52 PPL at $T=48$.  In the same setting, \SCSE{} also improves on the recurrent-step-conditioned adapter by 12.75, 12.53, and 11.10 PPL at the same depths.

\begin{table}[t!]
	\centering
	\scriptsize
	\resizebox{\columnwidth}{!}{
		\begin{tabular}{lrrr}
			\toprule
			                                           & \multicolumn{3}{c}{Test PPL}                                                       \\
			\cmidrule(lr){2-4}
			Method                                     & $T=8$                        & $T=24$                   & $T=48$                   \\
			\midrule
			Tuned adapter, $s=0.35$                    & $227.29\pm1.01$              & $249.40\pm1.57$          & $287.76\pm1.94$          \\
			Recurrent-step-conditioned adapter, 139.2M & $234.21\pm0.27$              & $255.10\pm0.42$          & $289.34\pm1.09$          \\
			\SCSE{}                                    & $\mathbf{221.46\pm0.59}$     & $\mathbf{242.57\pm0.70}$ & $\mathbf{278.24\pm1.53}$ \\
			\bottomrule
		\end{tabular}
	}
	\caption{Additional C4 scale evaluation with 136.5M shared-block models trained for 5k optimization steps.  The recurrent-step-conditioned control models have 139.2M parameters.  \SCSE{} has lower mean PPL than both adapter controls at every reported loop depth.}
	\label{tab:add-scale-real-corpus}
\end{table}

Appendix Table~\ref{tab:add-c4-long50k} shows that, in the 50k-step C4 scale evaluation, the absolute C4 PPL decreases from the 5k-step regime to the 50k-step regime.  At the largest reported point in the training loop-depth range, $T=8$, the tuned adapter remains competitive, and its PPL is 0.25 lower than \SCSE{}'s.  At deeper extra-loop depths, \SCSE{} has the lowest mean PPL, improving on the tuned adapter by 1.38 PPL at $T=32$, 5.42 PPL at $T=48$, and 35.83 PPL at $T=96$.  At the same deeper extra-loop depths, \SCSE{} also improves on the recurrent-step-conditioned adapter by 5.12, 7.52, and 0.72 PPL.  These C4 comparisons use the fixed protocol defined above.

\begin{table*}[t!]
	\centering
	\scriptsize
	\begin{tabular}{lrrrrr}
		\toprule
		                                                   & \multicolumn{5}{c}{Test PPL}                                                                                                             \\
		\cmidrule(lr){2-6}
		Setting                                            & $T=8$                        & $T=16$                   & $T=32$                   & $T=48$                   & $T=96$                   \\
		\midrule
		C4 tuned adapter, $s=0.35$, 136.5M, 50k            & $\mathbf{93.86\pm0.23}$      & $\mathbf{101.79\pm0.44}$ & $132.03\pm1.56$          & $178.50\pm3.71$          & $405.20\pm19.35$         \\
		C4 recurrent-step-conditioned adapter, 139.2M, 50k & $94.69\pm0.31$               & $103.46\pm1.00$          & $135.77\pm5.28$          & $180.60\pm10.67$         & $370.09\pm36.81$         \\
		C4 \SCSE{}, 136.5M, 50k                            & $94.11\pm0.28$               & $101.80\pm0.20$          & $\mathbf{130.65\pm1.15}$ & $\mathbf{173.08\pm2.68}$ & $\mathbf{369.37\pm9.67}$ \\
		\bottomrule
	\end{tabular}
	\caption{Longer C4 scale pretraining.  All models train for 50k optimization steps, corresponding to 204.8M training tokens per training run.  \SCSE{} is not uniformly best within the training loop-depth range.  The tuned adapter is slightly better at the largest reported point in that range, $T=8$, and the shallow extra-loop point $T=16$ is essentially tied.  At deeper extra-loop depths, \SCSE{} has the lowest mean PPL at $T=32$, $T=48$, and $T=96$.}
	\label{tab:add-c4-long50k}
\end{table*}

\subsubsection{When Extra Loops Help on Held-Out Web Text}
\label{app:extra-loop-help-real}

In-domain mean PPL often reaches its minimum near the boundary of the training loop-depth range, after which \SCSE{} maintains the strongest absolute shared-block quality even though the relative degradation slope is model dependent.  The held-out OpenWebText evaluation examines whether additional recurrent passes can also improve quality on a natural web-text distribution for WikiText-103-trained weights.  Under this distribution shift, the optimal loop depth can move beyond the training loop-depth range.

The held-out OpenWebText columns in Appendix Table~\ref{tab:heldout_transfer} show the same distribution-dependent extra-loop pattern without retabulating the main transfer benchmark.  Using unrounded means, at 50M, \SCSE{} improves from $3.685\pm0.067$k PPL at $T=8$ to $3.469\pm0.218$k PPL at $T=48$, while the recurrent-step-conditioned adapter degrades from $3.656\pm0.332$k to $4.015\pm0.361$k.  At 95.6M, \SCSE{} improves from $3.602\pm0.196$k to $3.421\pm0.305$k, while PPL increases by $0.071$k for the tuned adapter and by $0.140$k for the recurrent-step-conditioned adapter.  Thus, source-centered recurrence can benefit from extra loops under this web-text distribution shift, whereas the additive-source-conditioning controls degrade.

Appendix Tables~\ref{tab:add-causal-forcing},~\ref{tab:add-scale-real-corpus}, and~\ref{tab:add-c4-long50k}, together with the anchor-design results in Table~\ref{tab:add-anchor-design} and the held-out transfer results in Appendix Table~\ref{tab:heldout_transfer}, support the practical relevance of the source-centered design in these settings.  The learned anchor gives the lowest PPL among the tested source-centered anchor choices, and capacity-matched adapter baselines do not close the gap to \SCSE{}.  Among the tested intervention architectures, the separately trained additive forcing-subtraction control gives the largest deep extra-loop improvement, while mask-only changes leave PPL nearly unchanged.  Because the additive subtraction control uses a different injected source than the tuned adapter does, is trained independently, and doubles the body applications per step, we treat these controls as design comparisons rather than mechanistic attributions.

\subsection{Diagnostic WikiText-2 Tables}
\label{app:diagnostic-wt2}
Appendix Table~\ref{tab:ablations} reports results from the 22M, 1200-step WikiText-2 diagnostic.  These results compare stabilization, capacity, and source-centered controls, while the main text emphasizes stronger baselines on WikiText-103, larger-scale evaluations, adaptive-depth evaluations, and transfer comparisons.

\subsubsection{Stabilization and Capacity Controls}
Appendix Table~\ref{tab:ablations} compares \SCSE{} with stabilization and capacity controls at $T=8$ and at deep extra-loop depths.

First, periodic normalization reset is a strong baseline within the training loop-depth range, reaching $387.8\pm0.9$ PPL at $T=8$, but the reset baseline degrades sharply at deep extra-loop depths.  Second, \SCSE{} reaches $355.2\pm2.4$ at $T=8$ and has the lowest PPL at $T=24$ and $T=48$.

The replicated controls show that generic stabilization alone does not account for the source-centered gain.  Periodic normalization can improve the operating point within the training loop-depth range while suppressing useful computation at larger recurrent depth, whereas the source-centered and capacity-matched rows remain substantially stronger across the reported depths.

\begin{table}[t!]
	\centering
	\scriptsize
	\resizebox{\columnwidth}{!}{
		\begin{tabular}{lrrr}
			\toprule
			Method                                   & $T=8$ PPL              & $T=24$ PPL             & $T=48$ PPL             \\
			\midrule
			\multicolumn{4}{l}{Baseline and stabilization controls}                                                             \\
			Looped Transformer baseline              & $413.2\pm3.2$          & $426.9\pm3.7$          & $451.0\pm9.0$          \\
			Residual step scale $s=0.35$             & $414.0\pm2.0$          & $429.3\pm2.8$          & $456.3\pm7.8$          \\
			Learned scalar gate                      & $413.5\pm3.3$          & $427.0\pm3.9$          & $451.1\pm9.4$          \\
			Channel recurrent gate                   & $402.2\pm3.1$          & $421.2\pm3.0$          & $466.5\pm4.4$          \\
			Periodic normalization reset             & $387.8\pm0.9$          & $499.0\pm6.5$          & $1045.2\pm16.1$        \\
			\addlinespace[2pt]
			\midrule
			\multicolumn{4}{l}{Capacity and source-centered rows}                                                               \\
			Capacity-matched adapter                 & $360.0\pm1.3$          & $374.8\pm1.3$          & $397.3\pm0.9$          \\
			Parameter-matched unshared Transformer-4 & $358.1\pm2.9$          & not evaluated          & not evaluated          \\
			\SCSE{}                                  & $\mathbf{355.2\pm2.4}$ & $\mathbf{370.1\pm3.2}$ & $\mathbf{393.9\pm3.7}$ \\
			\bottomrule
		\end{tabular}
	}
	\caption{Stabilization and capacity controls on WikiText-2, 22M, 1200 steps.  All shared-block models use train-time loop-depth sampling from 1 to 8.  Periodic normalization reset is a strong baseline within the training loop-depth range but degrades at deep extra-loop depths.  \SCSE{} has the lowest PPL across the reported loop depths in the diagnostic table.}
	\label{tab:ablations}
\end{table}

\subsection{Forcing-Bias Intervention Ablations}
\label{sec:causal-forcing-ablation}

The forcing-bias intervention study compares a masked SC-Cond reference with three direct controls.  We use SC-Cond to denote this analysis family, distinct from the main \SCSE{} method, whose masked source-centered recurrence passes only the deviation to the recurrent block.  The masked SC-Cond reference keeps the source-conditioned block input and changes the branch within the zero-deviation threshold region through the mask.  One SC-Cond control keeps the same source-conditioned form but removes that mask.  For the deterministic evaluation map, a subtractive SC-Cond control uses the raw update $G_\theta(\Delta_t+c_{\rm sc}(e))-G_\theta(c_{\rm sc}(e))$ before applying the residual step scale, thereby removing the raw anchor response throughout the active update rather than only setting the branch within the zero-deviation threshold region.  Here $c_{\rm sc}(e)$ is a local SC-Cond conditioning vector.  The additive control uses the embedding source $e$ and subtracts its zero-deviation update.  The term $\tau_t$ denotes an optional recurrent-step conditioning term and is zero for controls without step conditioning:
\begin{align}
	\Delta_{t+1}
	 & =
	\Delta_t + s u_t^{\mathrm{fs}},\notag \\
	u_t^{\mathrm{fs}}
	 & =
	\Block_\theta(h^\star+\Delta_t+\alpha W_{\rm in}e+\tau_t)
	\notag                                \\
	 & -
	\Block_\theta(h^\star+\alpha W_{\rm in}e+\tau_t).
	\label{eq:additive-forcing-subtraction}
\end{align}
The additive forcing-subtraction control removes the exact anchor response from its deterministic evaluation map while retaining embedding-source additive injection.  The control uses the same learned anchor and initialization-module design as the adapter controls, but the architecture is not a subtraction-only version of the tuned adapter: Eq.~\ref{eq:additive-forcing-subtraction} injects $W_{\rm in}e$, whereas the tuned adapter injects $W_{\rm in}h^\star$.  Moreover, during training, the two body calls use independent dropout realizations, so their stochastic outputs need not cancel exactly at $\Delta_t=0$.  Exact cancellation applies to the dropout-disabled evaluation map measured in Appendix Table~\ref{tab:add-causal-forcing}.  The resulting forcing-subtraction architecture provides a constrained-injection comparison under the same WikiText-103 training protocol, related to the stable-injection direction represented by Parcae \citep{prairie2026parcae}.  The relation is at the level of injected-source stability rather than the proposed diagnostic: Parcae constrains residual-stream input injection to stabilize recurrent depth, whereas \SCSE{} reparameterizes the recurrent map to evolve anchor-relative deviations.  Because the forcing-subtraction control uses a different injected source than the tuned adapter does, is trained as its own architecture, and costs two shared-block body applications per recurrent step, the resulting architecture is an analysis comparison rather than an isolated frozen-weight intervention or a compute-matched replacement for \SCSE{}.  The first three rows with one shared-block body application per recurrent step in Appendix Table~\ref{tab:add-causal-forcing} reuse the matched 22M WikiText-103 reference PPL values from Appendix Table~\ref{tab:strong_followup} and Table~\ref{tab:stepcond_wt103}.  The remaining rows are the additional intervention controls.

\begin{table*}[t!]
	\centering
	\scriptsize
	\begin{tabular}{lrrrrr}
		\toprule
		                                   &                            & \multicolumn{3}{c}{Test PPL} &                                                                                             \\
		\cmidrule(lr){3-5}
		Method                             & Body applications per step & $T=8$                        & $T=24$                   & $T=48$                   & Pointwise-bias energy ratio at $t=47$ \\
		\midrule
		Tuned adapter, $s=0.35$            & 1                          & $156.43\pm0.58$              & $174.40\pm0.55$          & $204.05\pm1.43$          & $1.619$                               \\
		Recurrent-step-conditioned adapter & 1                          & $156.72\pm0.49$              & $175.61\pm0.68$          & $206.13\pm1.37$          & $1.276$                               \\
		\SCSE{}, $s=0.50$                  & 1                          & $155.14\pm0.61$              & $171.12\pm0.76$          & $200.10\pm1.31$          & $0$                                   \\
		SC-Cond, masked                    & 1                          & $154.66\pm0.53$              & $170.94\pm0.78$          & $200.35\pm0.98$          & $0$                                   \\
		SC-Cond, no mask                   & 1                          & $\mathbf{154.65\pm0.52}$     & $170.91\pm0.79$          & $200.31\pm0.98$          & $3.75$                                \\
		SC-Cond, subtractive               & 2                          & $166.57\pm0.52$              & $174.93\pm0.87$          & $183.07\pm1.27$          & $5.54\times 10^{-6}$                  \\
		Additive + forcing subtraction     & 2                          & $156.46\pm0.23$              & $\mathbf{158.86\pm0.45}$ & $\mathbf{160.46\pm0.66}$ & $0$                                   \\
		\bottomrule
	\end{tabular}
	\caption{Forcing-bias intervention controls on WikiText-103.  The ``Body applications per step'' column reports the number of physical shared-block body applications per logical recurrent step.  The pointwise-bias energy ratio is $R_{47}(e)$ from Eq.~\ref{eq:forcing-bias-ratio}, averaged over three stochastic replications; $R_{47}(e)$ can exceed one because state-dependent and pointwise terms can cancel.  The no-mask SC-Cond model has a large measured pointwise bias but similar PPL to the masked version.  In the deterministic evaluation map, the separately trained subtractive controls remove the anchor response throughout the active update and improve performance at the deepest reported extra-loop point, using two body applications per step.}
	\label{tab:add-causal-forcing}
\end{table*}

Appendix Table~\ref{tab:add-causal-forcing} separates masking within the zero-deviation threshold region from subtracting the anchor response throughout the deterministic active map.  Removing the mask from SC-Cond leaves a large measured pointwise-bias ratio but nearly unchanged PPL.  The separately trained subtractive architectures constitute a larger intervention and improve performance in the deep extra-loop regime.  At $T=48$, additive forcing subtraction reaches 160.46 PPL, compared with 204.05 for the tuned adapter and 200.10 for \SCSE{} with one body application per step.  The different injected source, independent training, and two-application cost prevent interpreting the tuned-adapter PPL gap as a subtraction-only or frozen-weight mediation estimate.  \SCSE{} instead uses a zero-preserving source-centered core with one body application.

\subsection{Additional Mechanism and Design Controls}
\label{app:add-exp}

We report an additional recurrent-step-modulation control for the source-centered comparison.  The forcing-bias intervention controls appear in the dedicated ablation above.  Unless a table states otherwise, the WikiText-103 ablations use the 22M, 5k-step shared-block protocol defined in the main WikiText-103 22M shared-block comparisons and the experimental setup.  This protocol uses a context length of 128, train-time loop depths sampled uniformly from $1$ through $8$, and evaluation at $T\in\{4,8,12,16,24,32,48\}$.  The 22M runs use AdamW with a learning rate of $3\times 10^{-4}$, 500 warmup steps, and a weight decay of 0.1.  Test PPL is reported as the mean $\pm$ one standard deviation.  Unless otherwise noted, these ablations use three stochastic replications.  Evaluations shared with the broader numerical suite follow the seed protocol.

For the non-\SCSE{} controls, we tune the residual step scale, the additive-source gain, and the recurrent-step-modulation hyperparameters under this WikiText-103 protocol; the selected values are then fixed across stochastic replications.  The stepwise feature-wise linear modulation (FiLM) adapter, denoted Step-FiLM, uses the modulation scale with the lowest deep extra-loop validation PPL in that sweep.

\subsubsection{Step-Modulated Adapter Control}
\label{app:step-modulated-control}

The Step-FiLM adapter control adds explicit recurrent-step modulation to the tuned additive adapter.  The Step-FiLM adapter uses a learned recurrent-step embedding and applies affine FiLM-style modulation before the shared block.  The modulation scale is selected under the same WikiText-103 protocol as the other adapter hyperparameters.  This control isolates recurrent-step modulation under matched parameter scale, data, loop-depth sampler, and optimizer.

With the selected modulation scale, the Step-FiLM adapter reaches $157.07\pm0.28$ PPL at $T=8$, $178.20\pm0.99$ at $T=24$, and $207.65\pm1.65$ at $T=48$.  The corresponding tuned-adapter, recurrent-step-conditioned, and \SCSE{} baselines are reported in Appendix Table~\ref{tab:strong_followup} and Table~\ref{tab:stepcond_wt103}.  Thus, adding either an explicit recurrent-step signal or a small recurrent-step-dependent FiLM path does not match the PPL of \SCSE{} in this setting.  The tuned adapter remains the lowest-PPL non-\SCSE{} 22M adapter among the controls with one shared-block body application per recurrent step.  \SCSE{} improves on the tuned adapter by 1.29 PPL at $T=8$, 3.28 PPL at $T=24$, and 3.95 PPL at $T=48$.

\subsection{Anchor-Coordinate Forcing-Bias Energy Ratios}
\label{app:forcing-bias-ratio-figure}

Appendix Figure~\ref{fig:forcing_bias_ratio} visualizes the pointwise forcing-bias energy ratios $R_t(e)$ reported in Table~\ref{tab:bias_decomp} across shared-block scales.  When the numerical floor in Eq.~\ref{eq:forcing-bias-ratio} is inactive, $R_t(e)$ is exactly $\norm{b_t(e)}_F^2/\norm{\Delta_{t+1}-\Delta_t}_F^2$.  We compute $R_t(e)$ for each measured batch with a fixed tensor shape and then average the resulting batch-level ratios over measured batches and stochastic replications.  The plot emphasizes that \SCSE{} has a zero pointwise ratio, while additive-source baselines retain order-one or larger late-step ratios.

\begin{figure*}[t!]
	\centering
	\includegraphics[width=0.98\textwidth]{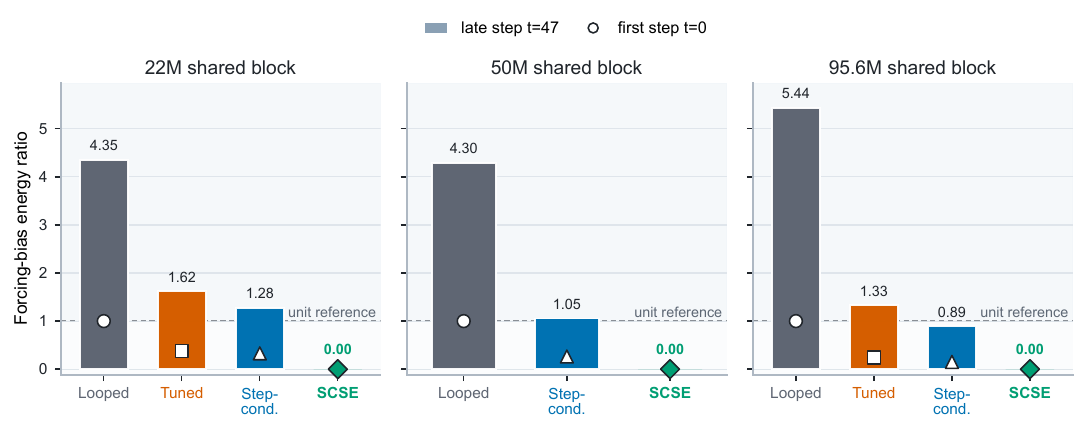}
	\caption{Pointwise forcing-bias energy ratio $R_t(e)$ from Eq.~\ref{eq:forcing-bias-ratio}, using the same method colors as Figure~\ref{fig:wt103-depth-response}.  The three plots group the 22M, 50M, and 95.6M WikiText-103 model weights.  Bars show the late recurrent step $t=47$, and open markers show the first recurrent step $t=0$ using the values in Table~\ref{tab:bias_decomp}.  The dashed line is a unit reference at which the pointwise zero-deviation forcing bias has the same squared energy as the measured net recurrent update.  The dashed line is not the \SCSE{} value.  \SCSE{} has a ratio of $0.00$ at both measured steps, while additive-source baselines retain nonzero pointwise terms whose ratios become order one or larger at late depth.  The ratio diagnoses the local presence of forcing bias relative to the net update, not trajectory accumulation or model quality.}
	\label{fig:forcing_bias_ratio}
\end{figure*}

\subsection{Secondary Deviation-Leak Diagnostic}
\label{app:secondary-condleak}
The main \SCSE{} recurrence in Eq.~\ref{eq:source} has no deviation-leak term ($\lambda_{\rm leak}=0$).  The leak diagnostic adds deviation leak and anchor-conditioned recurrent input to the same masked source-centered update.  The variant uses $\lambda_{\rm leak}=0.02$ and its own learned scalar conditioning gain $\kappa$, initialized to the same configured value $0.15$ as the baseline's separately learned additive-injection gain $\alpha$.  The two trained gains are not constrained to remain equal.  The variant is retained only to document the lower-gain regime near the boundary of the training loop-depth range.  The diagnostic can be slightly better near $T=8$, but the variant is consistently weaker once the loop budget is pushed deeper.

\begin{table*}[t!]
	\centering
	\scriptsize
	\begin{tabular}{llrrl}
		\toprule
		Evaluation                     & Metric     & Leak diagnostic        & \SCSE{}                & Interpretation                                         \\
		\midrule
		WikiText-103 22M fixed depth   & $T=8$ PPL  & $\mathbf{155.0\pm0.6}$ & $155.1\pm0.6$          & Essentially tied at the range boundary                 \\
		WikiText-103 22M fixed depth   & $T=48$ PPL & $227.5\pm2.9$          & $\mathbf{200.1\pm1.3}$ & \SCSE{} is stronger in deep extra-loop evaluation      \\
		WikiText-103 95.6M fixed depth & $T=8$ PPL  & $97.4\pm0.2$           & $\mathbf{96.9\pm0.3}$  & \SCSE{} is slightly stronger at scale                  \\
		WikiText-103 95.6M fixed depth & $T=48$ PPL & $143.0\pm1.5$          & $\mathbf{125.9\pm0.9}$ & \SCSE{} is much stronger in deep extra-loop evaluation \\
		\bottomrule
	\end{tabular}
	\caption{Secondary deviation-leak diagnostic.  The leak variant probes a lower-gain regime near the boundary of the training loop-depth range and is essentially tied with \SCSE{} near $T=8$ but substantially weaker at $T=48$.}
	\label{tab:secondary_condleak}
\end{table*}

\subsection{Beyond Pointwise Bias: Operating Regimes}
\label{app:after-bias-removal}

For the local-sensitivity diagnostics, random-direction gain is the finite-difference directional estimate
\begin{align*}
	\frac{\norm{\Tmap_t(\Delta+\epsilon_{\rm fd} v;e)-\Tmap_t(\Delta;e)}_F}{\epsilon_{\rm fd}},
\end{align*}
with $\norm{v}_F=1$ and $\epsilon_{\rm fd}=10^{-3}$, averaged over four random probe directions, measured batches, and stochastic replications.  Appendix Table~\ref{tab:operating_point_mechanism} reports the corresponding 22M diagnostics for \SCSE{} and the tuned-adapter control.

The operating-point diagnostics complement the pointwise-bias measurement.  \SCSE{} has zero pointwise bias, remains near-neutral under the sampled late-step gain diagnostic, and delivers the best deep extra-loop quality, while the tuned adapter retains nonzero forcing bias despite its lower anchor energy.

The same pattern persists across the broader evaluations.  Recurrent-step-conditioned adapters become strong baselines at 50M and 95.6M, but they still trail \SCSE{} at large in-domain depths and on held-out OpenWebText, held-out C4, and LAMBADA.

Alongside the pointwise-bias diagnostic, local-sensitivity measurements compare the learned one-step maps under the 22M-parameter, 5k-step WikiText-103 protocol.  The looped Transformer baseline starts with a $t=0$ random-direction gain of $3.98\pm0.05$.  The random-direction-gain mean indicates that the sampled perturbations are strongly amplified on average at that operating point; the statistic neither estimates the largest singular value nor certifies expansion in every direction.  The tuned adapter reduces the random-direction gain to $1.135\pm0.002$ and the recurrent-step-conditioned adapter to $1.135\pm0.004$ at $t=0$, with both becoming essentially neutral at deeper steps.

\SCSE{} starts slightly lower at $t=0$ with $1.118\pm0.003$ and settles to a slightly above-unit late-step sampled-gain regime, reaching $1.00108\pm0.00003$ at $t=47$.  Appendix Table~\ref{tab:secondary_condleak} reports the lower-gain leak diagnostic separately because the diagnostic is useful mainly near the boundary of the training loop-depth range and is weaker at larger loop depths.

These 22M diagnostics compare sampled local directional gains; they neither estimate the Jacobian operator norm nor establish local or global contraction.  They show that the source-centered reparameterization neutralizes zero-deviation forcing while retaining nonzero recurrent sensitivity away from the anchor.

The fixed-depth evaluations across trained and extra-loop depths clarify the operating range of the method.  \SCSE{} usually provides the best deep extra-loop PPL.

The adaptive-depth evaluations in Appendix Tables~\ref{tab:adaptive_wt103} and~\ref{tab:adaptive_scale} sharpen the operating picture: \SCSE{} remains the best reported loop-budgeted shared-block method at target mean loop budgets of 8, 16, and 24 across the 22M, 50M, and 95.6M comparisons.

\subsection{Operating-Point Diagnostics}
\label{app:operating-point-diagnostics}
Appendix Table~\ref{tab:operating_point_mechanism} reports the diagnostics referenced in the main mechanistic analysis.  The table keeps the source-centered comparison focused on the reported \SCSE{} method.

\begin{table}[t!]
	\centering
	\scriptsize
	\resizebox{\columnwidth}{!}{
		\begin{tabular}{lrrrr}
			\toprule
			                        & PPL                    & \multicolumn{3}{c}{Late-step diagnostics ($t=47$)}                              \\
			\cmidrule(lr){2-2}\cmidrule(lr){3-5}
			Method                  & $T=48$                 & Gain                                               & Bias ratio & Anchor energy \\
			\midrule
			Tuned adapter, $s=0.35$ & $204.1\pm1.4$          & $1.00035$                                          & $1.619$    & $9.84$        \\
			\SCSE{}, $s=0.50$       & $\mathbf{200.1\pm1.3}$ & $1.00108$                                          & $0.000$    & $18.10$       \\
			\bottomrule
		\end{tabular}
	}
	\caption{Mechanism diagnostics on WikiText-103, 22M, 5k steps.  The source-centered row has zero pointwise response and a zero-preserving raw core, so the remaining comparison is governed by the sampled random-direction gain regime and how much state-dependent computation survives at late recurrent steps.  Anchor energy is the mean-squared anchor-relative deviation at the late state, not a lower-is-better quality metric.}
	\label{tab:operating_point_mechanism}
\end{table}

\subsection{Related Work}
\label{app:related-work}

\paragraph{Looped and recurrent-depth Transformers.} \citet{dehghani2019universal} introduced recurrent self-attention with optional adaptive computation, while studies of looped Transformers and recurrent algorithms examine recurrence as a programmable, latent-reasoning, or easy-to-hard extrapolation mechanism \citep{giannou2023looped,schwarzschild2021learn,bansal2022algorithm,fan2025looped,saunshi2025latent}.  Recent looped Transformers and elastic-depth variants revive recurrent depth as a test-time compute axis.  Huginn demonstrates arbitrary latent-depth unrolling \citep{geiping2025huginn}.  Ouro scales up open looped Transformer pretraining with learned depth allocation and entropy regularization \citep{zhu2025ouro}.  LoopFormer studies budget-conditioned loop trajectories \citep{jeddi2026loopformer}, and Hyperloop studies parameter-efficient looped Transformer training \citep{zeitoun2026hyperloop}.  Depth-recurrent Transformer work on compositional generalization is another nearby algorithmic-depth reference point \citep{chen2026thinking}.

\paragraph{Loop-depth failures and recurrent stability.} Several recent studies share our broad concern that hidden-state dynamics can become unreliable when recurrent depth is extended, but they target different mechanisms.  Parcae models the looped residual stream as a forced dynamical system and constrains repeated injection for stable scaling \citep{prairie2026parcae}.  \citet{yang2026stars} regularize the recurrent Jacobian's spectral radius, while \citet{park2026loopus} learn input-dependent selective gates to mitigate hidden-state drift.  \citet{sharma2026readout} identify a readout blind spot through which scale-invariant supervision can leave the hidden-state norm uncontrolled.  These approaches share the broad phenomenon of recurrent instability or drift with this paper.  They do not, to our knowledge, define and measure the anchor-response quantity $b_t(e)=\Tmap_t(0;e)$ or enforce its pointwise removal as the primary architectural condition.  Conversely, the local anchor-consistency condition is not a general stability theorem; the condition neither constrains the full Jacobian nor subsumes the readout blind spot.

Input injection itself is not intrinsically harmful.  \citet{blayney2026mechanistic} find that injection can encourage fixed-point convergence in looped reasoning models, and path-independent equilibrium models use input injection to improve test-time-compute generalization \citep{anil2022path}.  These results are consistent with our contractive case, in which bounded forcing produces a bounded response.  Our narrower claim concerns forcing components that recurrent propagation preserves and state-dependent updates do not cancel; the pointwise anchor-response diagnostic and the exact bias-subtraction counterfactual used to state those conditions are the points of difference.

\paragraph{Cross-layer parameter sharing.} Cross-layer parameter sharing itself is not new.  Cross-layer-sharing approaches reduce parameter memory in bidirectional encoder models, recurrently stacked translation layers, Trellis-style sequence models with weight tying across depth and direct input injection, and broader parameter-sharing settings \citep{lan2020albert,dabre2019recurrent,bai2019trellis,takase2023lessons}.  The models studied here instead repeatedly apply one recurrent block to a changing state distribution and then evaluate that same set of trained weights at loop depths beyond the training loop-depth range.  Applying the same trained weights beyond the training loop-depth range introduces a dynamical evaluation concern beyond parameter sharing alone.  The same parameters must act on a sequence of hidden-state distributions whose anchor-relative deviation, zero-deviation forcing bias, and measured random-direction gain can drift with recurrent step.  A deeper unshared Transformer can specialize each layer to a particular depth index, while the shared block in a looped Transformer must learn a single transition rule that remains useful under both trained and extrapolated loop depths.

\paragraph{Residual stabilization and dynamical views.} Our stabilization view is related to interpretations of highway and residual networks as iterative or recurrent estimators \citep{srivastava2015highway,liao2016bridging,greff2017highway,jastrzebski2018residual}.  The stabilization view is also related to residual scaling and initialization methods such as Fixup, ReZero, DeepNorm, and LayerScale \citep{zhang2019fixup,bachlechner2021rezero,wang2022deepnet,touvron2021going}.  Other connections include Transformer residual-dependency analyses \citep{liu2020understanding}, normalization choices and pre-norm analyses \citep{ba2016layer,zhang2019rmsnorm,nguyen2019transformers,xiong2020layer}, stable recurrent-model analyses \citep{pascanu2013difficulty,miller2019stable,chang2019antisymmetric,erichson2021lipschitz}, and dynamical-systems and continuous-depth views of residual networks \citep{haber2017stable,chen2018neuralode}.  These approaches control update scale, normalization, residual dependency, or local spectral behavior.  \SCSE{} is complementary because the method reparameterizes the recurrent vector field around an input-conditioned anchor rather than directly controlling these stabilization quantities.

\paragraph{Equilibrium and adaptive computation.} Equilibrium and adaptive-computation models provide another nearby comparison.  Deep Equilibrium Models solve for an implicit fixed point \citep{bai2019deq}, and path-independent equilibrium models regularize trajectories toward input-conditioned equilibria while targeting generalization across iteration counts \citep{anil2022path}.  Adaptive computation, probabilistic halting, depth-adaptive Transformers, depth-on-demand pruning, early-exit LM methods, and token-level dynamic-depth routing decide how much computation to spend per example or token \citep{graves2016act,banino2021pondernet,elbayad2020depth,fan2020layerdrop,schuster2022calm,raposo2024mixture}.  Here, we study explicit recurrent unrolls, report fixed-depth and offline loop-budgeted evaluations, and use adaptive-depth to show that source-centered shared blocks expose a reusable inference-time depth knob.

\paragraph{Positioning of this study.} All these systems motivate the recurrent-depth axis while varying in scale, data, loop placement, training objective, adaptive depth allocation, and stability parameterization.  Within this broader recurrent-depth design space, we isolate the full-state, shared-block regime in which every token state is updated by the same block over recurrent depth.  The stability work above shares the broad problem of hidden-state evolution becoming unreliable with extra loops.  Our narrower contribution is to define the anchor-dependent pointwise diagnostic $b_t(e)$, distinguish the pointwise diagnostic from the raw pre-mask response and the finite-horizon counterfactual response, state sufficient propagation and loss-alignment conditions under which the latter becomes harmful, and test a source-centered reparameterization with a zero-preserving core and an exact masked anchor boundary.  The matched controls cover ingredients most likely to explain the gain, including capacity-matched adapters, recurrent-step conditioning, learned step modulation, residual-scale tuning, parameter-matched unshared controls, forcing-subtraction architectures, and absolute unshared-depth references.

\subsection{Limitations}
The experiments isolate controlled shared-block LM settings rather than frontier-scale LM pretraining.  The largest source-centered model has 136.5M trainable parameters, and the largest matched control has 139.2M trainable parameters.  The longest direct C4 suite uses 204.8M training tokens per trained model.  Billion-parameter scaling-law behavior and billion-token web-corpus training therefore remain open.

Zero-deviation forcing bias is defined relative to a chosen, model-specific anchor.  The pointwise condition $b_t(e)=0$ says that zero deviation is a fixed point of $\Tmap_t(\cdot;e)$, equivalently that this anchor is a one-step fixed point at recurrent step $t$.  This pointwise condition is not invariant under changes to the anchor and does not imply that the anchor is an optimal endpoint.  For a masked map, the pointwise response can also differ from the raw active-branch response, as the no-mask SC-Cond control demonstrates.  \SCSE{} removes both responses in the reported zero-preserving core, but \SCSE{} also changes the learned recurrent vector field by applying the shared core to deviations, so the method should be viewed as an architectural reparameterization rather than a lossless coordinate rewrite.

The exact finite-horizon theorem, its contraction and coherence corollaries, and the task-loss proposition give sufficient conditions for a bias-subtraction response to remain bounded, accumulate, and increase task loss.  These results do not constitute a global theorem for stability or PPL.  \SCSE{} removes the local pointwise and raw anchor responses in the reported core but does not constrain off-anchor Jacobians, the state-dependent residual, or readout sensitivity.  Figure~\ref{fig:zdfb-empirical} is a one-replication geometric case study that uses separate projections for the two plots.  Table~\ref{tab:bias_decomp} provides replicated pointwise measurements.  Appendix Table~\ref{tab:add-causal-forcing} shows both that a nonzero pointwise response can be benign in one control and that separately trained subtraction architectures can substantially improve deep extra-loop PPL in others.  For the subtractive controls, exact cancellation refers to the dropout-disabled evaluation map; during training, independent dropout realizations in the two body calls need not cancel.  Because these controls are retrained, require an extra body application, and, in the additive case, use a different injected source than the tuned adapter does, they do not provide a subtraction-only or frozen-weight causal estimate.  Establishing how often the sufficient propagation and readout-alignment conditions hold across architectures, anchors, and scales remains open.
\end{document}